\documentclass[10pt,twocolumn,letterpaper]{article}

\usepackage{cvpr}              %

\usepackage{url}
\usepackage[utf8]{inputenc}
\usepackage{booktabs} %
\usepackage{amsfonts} %
\usepackage{nicefrac} %
\usepackage{microtype} %
\usepackage{natbib}
\usepackage{gensymb}
\usepackage{color}
\usepackage{colortbl}
\usepackage{caption}
\usepackage{subcaption}
\usepackage{amsthm}
\usepackage{mathrsfs}
\usepackage{amsmath}
\usepackage{amssymb}
\usepackage{mathtools}
\usepackage{wrapfig}
\usepackage{soul}
\usepackage[ruled,longend, algo2e]{algorithm2e}
\usepackage[most]{tcolorbox}

\usepackage{graphicx}
\usepackage[dvipsnames]{xcolor}
\usepackage{svg}
\usepackage{diagbox}
\usepackage{algorithm}
\usepackage{algorithmic}
\usepackage{hhline}
\usepackage{multirow}
\usepackage{wrapfig,lipsum,booktabs}
\usepackage{bbm}
\usepackage{dsfont}

\definecolor{codegreen}{rgb}{0,0.6,0}
\definecolor{codegray}{rgb}{0.5,0.5,0.5}
\definecolor{codepurple}{rgb}{0.58,0,0.82}
\definecolor{backcolour}{rgb}{0.95,0.95,0.92}
\definecolor{mediumtealblue}{rgb}{0.0, 0.33, 0.71}
\definecolor{darkpastelgreen}{rgb}{0.01, 0.75, 0.24}
\definecolor{azure}{rgb}{0.0, 0.5, 1.0}
\newcommand{\cellbest}{\cellcolor{orange!35}}
\newcommand{\cellsecond}{\cellcolor{orange!10}}

\newcommand{\highlighteqn}[2]{\raisebox{0.0ex}{\colorbox{#1}{$\displaystyle #2$}}}

\newcommand{\highlightalg}[2]{\colorbox{#1}{\parbox{\dimexpr\linewidth-2\fboxsep}{#2}}}

\newcommand{\emojiice}{\raisebox{-0.2\height}{\includegraphics[width=1em]{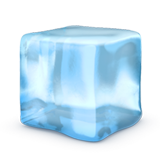}}}
\newcommand{\emojifire}{\raisebox{-0.1\height}{\includegraphics[width=1em]{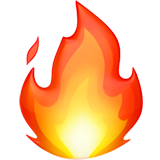}}}

\newcommand{\Mat}{\boldsymbol}
\newcommand{\Set}{\mathcal}

\newcommand{\real}{\mathbb{R}}

\newcommand{\nat}{\mathbb{N}}

\newcommand{\bigO}{\mathcal{O}}

\newcommand{\wt}[1]{\widetilde{#1}}
\newcommand{\wb}[1]{\overline{#1}}

\newcommand{\cL}{\mathcal{L}}

\newtheorem{theorem}{Theorem}

\newtheorem{lemma}[theorem]{Lemma}

\DeclareMathOperator{\trace}{tr}

\DeclareMathOperator{\mean}{\mathbb{E}}

\DeclareMathOperator{\gauss}{\mathcal{N}}

\definecolor{cvprblue}{rgb}{0.21,0.49,0.74}
\usepackage[pagebackref,breaklinks,colorlinks,allcolors=cvprblue]{hyperref}

\def\paperID{5109} %
\def\confName{CVPR}
\def\confYear{2025}

\title{Steepest Descent Density Control for Compact 3D Gaussian Splatting}

\author{Peihao Wang\textsuperscript{1}\footnotemark[1]\hspace{0.4em}\footnotemark[2], Yuehao Wang\textsuperscript{1}\footnotemark[1], Dilin Wang\textsuperscript{2}, Sreyas Mohan\textsuperscript{2}, Zhiwen Fan\textsuperscript{1}, Lemeng Wu\textsuperscript{2}, \\ Ruisi Cai\textsuperscript{1}, Yu-Ying Yeh\textsuperscript{2},
Zhangyang Wang\textsuperscript{1}, Qiang Liu\textsuperscript{1}, Rakesh Ranjan\textsuperscript{2} \\
{\textsuperscript{1}The University of Texas at Austin, \textsuperscript{2}Meta Reality Labs}\\
{\tt\small\{peihaowang, yuehao, zhiwenfan, ruisi.cai, atlaswang\}@utexas.edu,} \\
{\tt\small lqiang@cs.utexas.edu, \{wdilin, sreyasmohan, lmwu, yyyeh, rakeshr\}@meta.com} \\
\tt\small \href{https://vita-group.github.io/SteepGS/}{vita-group.github.io/SteepGS}
}

\begin{document}

\twocolumn[{%
\renewcommand\twocolumn[1][]{#1}%
\maketitle
\begin{center}
\centering
\begin{tabular}{lr}
\includegraphics[trim= 10 12 10 10,clip,width=0.68\linewidth]{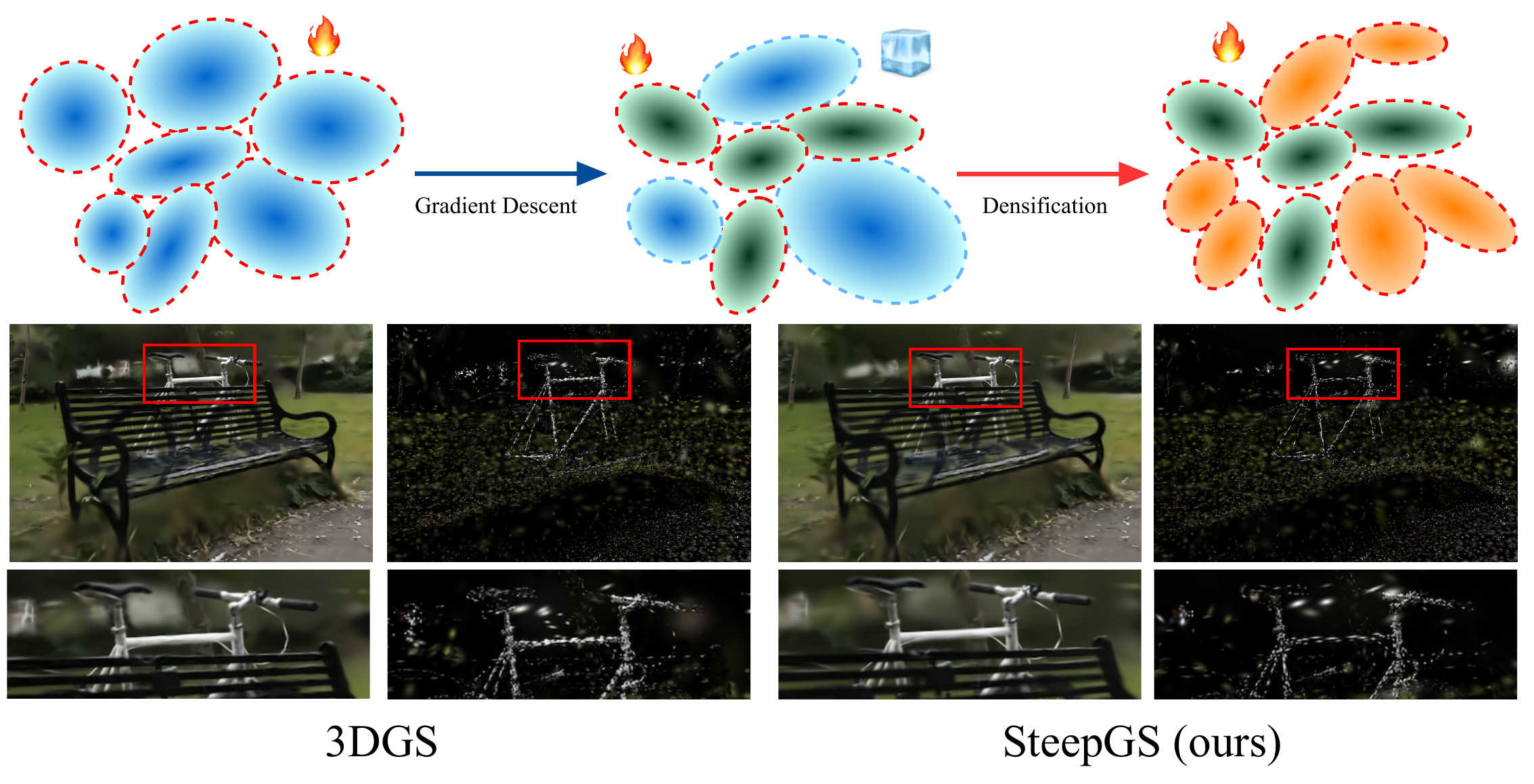} &
\includegraphics[trim= 20 1 14 20,clip,width=0.3\linewidth]{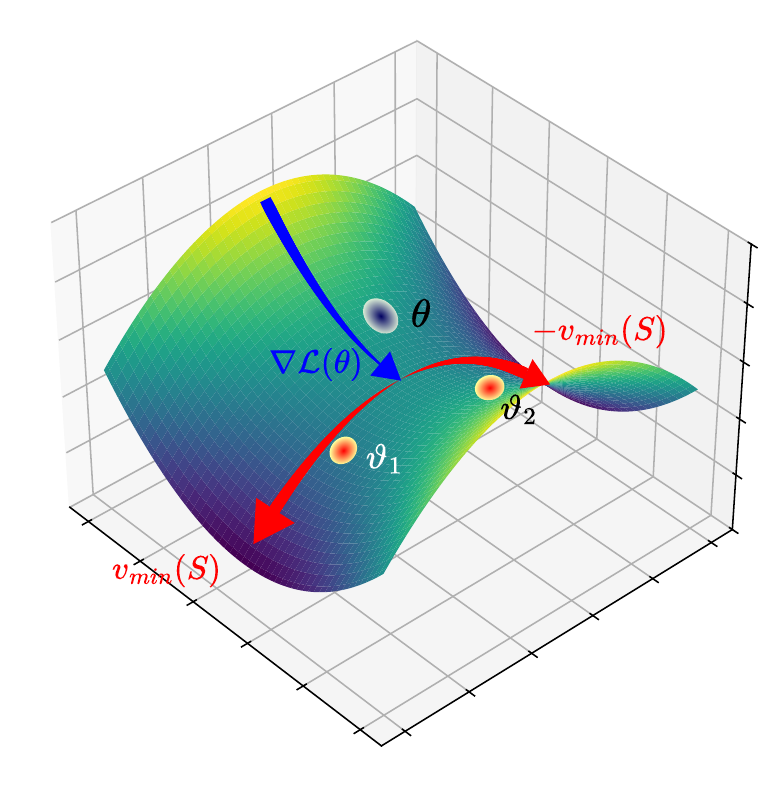}
\end{tabular}
\captionsetup{hypcap=false}
\vspace{-1em}
\captionof{figure}{
We theoretically investigate density control in 3DGS. As training via gradient descent progresses, many Gaussian primitives are observed to become stationary while failing to reconstruct the regions they cover (\eg the \textcolor{cyan}{cyan}-colored blobs in the top-left figure marked with \protect\emojiice).
From an optimization-theoretic perspective (see figure on the right), we reveal that these primitives are trapped in saddle points, the regions in the loss landscape where \textit{\textcolor{blue!75}{gradients}} are insufficient to further reduce loss, leaving parameters sub-optimal locally.
To address this, we introduce \textbf{SteepGS}, which efficiently identifies Gaussian points located in saddle area, splits them into two off-springs, and displaces new primitives along the \textit{\textcolor{red!65}{steepest descent directions}}.
This restores the effectiveness of successive gradient-based updates by escaping the saddle area (\eg the \textcolor{orange}{orange}-colored blobs in the top-left figure marked with \protect\emojifire~become optimizable after densification).
As shown in the bottom-left visualization, SteepGS achieves a more compact parameterization while preserving the fidelity of fine geometric details.
}
\label{fig:teaser}
\end{center}%
}]

\renewcommand{\thefootnote}{\fnsymbol{footnote}}
\footnotetext[1]{Equal contribution.}
\footnotetext[2]{Work done during an internship with Meta Reality Labs.}
\renewcommand{\thefootnote}{\arabic{footnote}}

\begin{abstract}
3D Gaussian Splatting (3DGS) has emerged as a powerful technique for real-time, high-resolution novel view synthesis.
By representing scenes as a mixture of Gaussian primitives, 3DGS leverages GPU rasterization pipelines for efficient rendering and reconstruction.
To optimize scene coverage and capture fine details, 3DGS employs a densification algorithm to generate additional points.
However, this process often leads to redundant point clouds, resulting in excessive memory usage, slower performance, and substantial storage demands--posing significant challenges for deployment on resource-constrained devices.
To address this limitation, we propose a theoretical framework that demystifies and improves density control in 3DGS.
Our analysis reveals that splitting is crucial for escaping saddle points.
Through an optimization-theoretic approach, we establish the necessary conditions for densification, determine the minimal number of offspring Gaussians, identify the optimal parameter update direction, and provide an analytical solution for normalizing off-spring opacity.
Building on these insights, we introduce \textbf{SteepGS}, incorporating \textit{steepest density control}, a principled strategy that minimizes loss while maintaining a compact point cloud. SteepGS achieves a $\sim$ 50\% reduction in Gaussian points without compromising rendering quality, significantly enhancing both efficiency and scalability.
\end{abstract}

\section{Introduction}
\label{sec:intro}

3D Gaussian Splatting (3DGS) \cite{kerbl20233d}, as a successor to Neural Radiance Fields (NeRF) \cite{mildenhall2021nerf} for novel view synthesis, excels in delivering impressive view synthesis results while achieving real-time rendering of large-scale scenes at high resolutions.
Unlike NeRF’s volumetric representation, 3DGS represents radiance fields of 3D scenes as a mixture of Gaussian primitives, each defined by parameters such as location, size, opacity, and appearance \cite{zwicker2002ewa}.
By utilizing the rasterization pipeline integrated with GPUs, this approach allows for ultra-efficient rendering and backpropagation, significantly accelerating scene reconstruction and view inference.

At the core of 3DGS is an alternating optimization process that enables accurate approximation of complex scenes using Gaussian primitives.
Starting with a precomputed sparse point cloud as the initialization, 3DGS cycles between standard gradient-based photometric error minimization to refine Gaussian parameters, and a tailored Adaptive Density Control (ADC) algorithm \cite{kerbl20233d} to adjusts the number of Gaussian points.
During the densification phase, ADC identifies a set of heavily optimized points and splits each into two offspring, assigning distinct parameter updates based on their absolute sizes to ensure comprehensive scene coverage and capture fine geometric details.
However, this reconstruction pipeline in 3DGS often produces excessively large point clouds, causing increased memory usage, slower rendering speed, and significant disk overhead.
This issue poses a critical bottleneck for deployment on resource-constrained devices such as mobile phones and VR headsets.

While post-hoc pruning and quantization-based compression algorithms have been widely used to address this challenge \cite{fan2023lightgaussian, lee2024compact, navaneet2023compact3d, morgenstern2023compact, niedermayr2024compressed, wang2024end, papantonakis2024reducing, hanson2024pup}, there are only few approaches that directly tackle this problem through the densification process.
Optimizing the densification phase could potentially yield compact Gaussian point clouds for faster rendering while simultaneously reducing training costs.
Some prior works have attempted to revise the density control algorithm using heuristics, such as modifying the splitting criteria \cite{bulo2024revising} or generating new Gaussians by sampling from the opacity distribution \cite{kheradmand20243d}. However, the densification process is not well understood,  and as a result, the existing solutions only achieve very limited improvement as they primarily rely on heuristics.

In this paper, we theoretically demystify the density control algorithms for 3DGS through the lens of non-convex optimization.
Our analysis characterizes the loss behavior after splitting by introducing a novel matrix, termed the \textit{splitting matrix}, which links the first-order gradient of each Gaussian to its corresponding Hessian.
We prove that the splitting operation is crucial in 3DGS for escaping saddle points.
However, not all points benefit from densification.
Specifically, we demonstrate that splitting only reduces the loss if the associated splitting matrix is not positive semi-definite. This insight highlights the importance of splitting in density control algorithms, offering a deeper optimization-theoretic perspective to complement the existing geometric understanding.

Based on these theoretical investigations, we further draw the following affirmative conclusions:
\textit{(i)} Splitting each Gaussian into \textit{two} offspring is sufficient to achieve the optimal descent on loss while ensuring a controlled growth rate of the number of points.
\textit{(ii)} The magnitude of the off-spring Gaussians should be halved to preserve the local density.
\textit{(iii)} To achieve the steepest descent in loss after splitting, the new Gaussians should be displaced along the positive and negative directions of the eigenvector corresponding to the least eigenvalue of the splitting matrix.
We consolidate these findings into a principled splitting strategy, termed \textit{Steepest Density Control (SDC)}, which provably maximizes loss reduction while minimizing the number of yielded Gaussian off-springs.

We further demonstrate that Steepest Density Control (SDC) can be efficiently implemented and seamlessly integrated into the existing 3DGS CUDA kernel.
To compute the splitting matrix, we propose a parallel algorithm that leverages the closed-form Hessian for each Gaussian, combined with gradient information reused from backpropagation.
We term this enhanced 3DGS-based reconstruction system \textbf{SteepGS}.
Empirically, SteepGS achieves over a 50\% reduction in Gaussian points while maintaining high rendering quality, significantly improving memory efficiency and rendering speed.

\section{Related Work}

\subsection{Efficient Scene Representations}

Recent advancements in neural scene representations have transformed view synthesis. Neural Radiance Fields (NeRF)~\cite{mildenhall2021nerf} introduced a method for synthesizing photorealistic views by optimizing a continuous volumetric scene function using sparse input views, however their high computational cost limits their applications. Building on NeRF, several methods have attempted to reduce the computational like, InstantNGP~\cite{muller2022instant} by employing a multiresolution hash encoding, TensoRF~\cite{chen2022tensorf} by modelling radiance field as a 4D tensor with compact low-rank factorization, Generalizable NeRFs~\cite{wang2021ibrnet, chen2021mvsnerf, wang2022attention} by leveraging attention to decode assosciations between multiple views, and Plenoxel\cite{yu2021plenoctrees, fridovich2022plenoxels} by using a view-dependent sparse voxel model.  Light Field Networks~\cite{sitzmann2021light} and Surface Based Rendering~\cite{liu2019soft, shen2021deep, yariv2020multiview}, and Point Based Rendering~\cite{yifan2019differentiable, aliev2020neural, xu2022point} have also made significant contributions by proposing novel neural scene representations and differentiable rendering techniques. Unlike these methods, 3DGS~\cite{kerbl20233d} offers an alternative representation by modeling scenes with learnable anisotropic Gaussian kernels, enabling fast rendering through point-based rasterization.

\subsection{Compact 3D Gaussian Splatting}

The original 3D Gaussian Splatting~\cite{kerbl20233d} optimization tends to represent the scene with redundant GS. Several works have been proposed to obtain a more compact scene with distilled or compact representation on GS attributes~\cite{fan2023lightgaussian,lee2024compact}, improved pruning~\cite{kheradmand20243d,fan2023lightgaussian}, or even efficient densification~\cite{lee2024compact,kheradmand20243d,bulo2024revising, mallick2024taming}. Among the densification approaches, Lee~\etal~\cite{lee2024compact} proposes a learnable masking strategy during densification while using compact grid-based neural field and codebook to represent color and geometry. 3DGS-MCMC~\cite{kheradmand20243d} rewrites densification and pruning as deterministic state transition of MCMC samples. Revising-3DGS~\cite{bulo2024revising} proposes pixel-error driven density control for densification. Taming-3DGS~\cite{mallick2024taming} applies score-based densification with predictable budget control. Instead, our method takes partial Hessian information to achieve steepest density control.

\subsection{Neural Architecture Splitting}
In the field of neural architecture search, a promising direction involves starting with a small seed network and gradually expanding its size based on target performance and user constraints. This approach, known as neural architecture growth or splitting, was initially proposed by Net2Net~\cite{chen2015net2net}, which introduced the concept of growing networks by duplicating neurons or layers with noise perturbation.
Follow-up works~\cite{evci2022gradmax, wang2023learning, chen2021bert2bert, wu2022residualmixtureexperts} extend network module-wise growth to enhance model performance across various applications.
The most relevant prior work is perhaps S2D~\cite{wu2019splitting, wang2019energy, wu2020steepest}, which refines the neuron splitting process by calculating subsets of neurons to split and applying a splitting gradient for optimal updating of the off-springs.
After that, Firefly~\cite{wu2020firefly} introduces a first-order gradient-based approach to approximate the S2D split matrix, accelerating the splitting process and making the splitting method scalable.
Although theoretically related, these methods are not directly applicable to 3DGS optimization.
Our investigation into splitting schemes in 3DGS aims to address this gap.

\section{Preliminaries: 3D Gaussian Splatting}
\label{sec:prelim}

In this section, we briefly review the reconstruction pipeline of 3DGS, along the way, introducing necessary notations.

\vspace{-1em}
\paragraph{Scene Representation.}
The core idea of 3D Gaussian Splatting (3DGS) \cite{kerbl20233d} is to approximate the radiance field \cite{mildenhall2021nerf} of complex scenes via a mixture of Gaussian primitives \cite{zwicker2002ewa}.
Formally, consider one scene that is represented by a set of $n$ primitives with parameters $\Mat{\theta} = \{\Mat{\theta}^{(i)} \in \Theta \}_{i=1}^{n}$, where $\Mat{\theta}^{(i)} \triangleq (\Mat{p}^{(i)}, \Mat{\Sigma}^{(i)}, o^{(i)}, \Mat{c}^{(i)})$, $\Mat{p}^{(i)} \in \real^3$ denotes the central position, $\Mat{\Sigma}^{(i)} \in \mathbb{S}_{+}^{3 \times 3}$ is the positive semi-definite covariance matrix, often re-parameterized via a quaternion plus a scaling vector, $o^{(i)} \in [0, 1]$ denotes the magnitude of the Gaussian function, often geometrically interpreted as the opacity value.
Moreover, each Gaussian point is associated with color attributes $\Mat{c}^{(i)} \in \real^{3}$, which are stored as spherical harmonics coefficients and converted to RGB values at the rendering time.
The entire density field of the scene $\alpha: \real^3 \rightarrow \real$ is expressed as a combination of all primitives $\alpha(\Mat{\xi}) = \sum_{i=1}^{N} \sigma(\Mat{\xi}; \Mat{\theta}^{(i)})$, where each Gaussian primitive $\sigma(\cdot ; \Mat{\theta}^{(i)}): \real^3 \rightarrow \real$ contributes as a scaled Gaussian kernel on 3D point $\Mat{\xi} \in \real^3$ \cite{zwicker2002ewa}:
\begin{align}
\label{eqn:sigma_3D}
\sigma(\Mat{\xi} ; \Mat{\theta}^{(i)}) =
o^{(i)} \gauss(\Mat{\xi}; \Mat{p}^{(i)}, \Mat{\Sigma}^{(i)}).
\end{align}
Here $\gauss(\Mat{\xi}; \Mat{p}, \Mat{\Sigma}) = \exp(-\frac{1}{2}(\Mat{\xi} - \Mat{p})^\top \Mat{\Sigma}^{-1} (\Mat{\xi} - \Mat{p}))$ denotes an unnormalized Gaussian function.

\vspace{-1em}
\paragraph{Rasterization.}
To display a pixel on the screen, 
\citet{zwicker2002ewa} shows that volume rendering \cite{max1995optical} can be realized by first projecting each primitive individually and then compositing each primitive via a back-to-front $\alpha$-blending.
Formally, suppose the pixel location is $\Mat{x} \in \real^2$ and world-to-camera projection is $\Pi: \real^3 \rightarrow \real^2$, which encompasses both camera extrinsics and intrinsics.
To obtain the rendered color for pixel $\Mat{x}$, 3DGS first sorts points according to view-dependent depth and then adopts an efficient rasterization process formulated as follows:
\begin{align}
\label{eqn:alpha_blend}
\Mat{C}_{\Pi}(\Mat{x}; \Mat{\theta}) = \sum_{i=1}^{n} \Mat{c}_i T_i \sigma_{\Pi}(\Mat{x}; \Mat{\theta}^{(i)}),
\end{align}
where $T_i = \prod_{j=1}^{i-1} (1-\sigma_{\Pi}(\Mat{x}; \Mat{\theta}^{(j)}))$, and $\sigma_{\Pi}(\Mat{x}; \Mat{\theta}^{(i)})$ has the analytical form:
\begin{align}
\label{eqn:sigma_2D}
\sigma_{\Pi}(\Mat{x}; \Mat{\theta}^{(i)}) = o^{(i)} \gauss\left(\Mat{x}; \Pi(\Mat{p}^{(i)}), \Pi(\Mat{\Sigma}^{(i)})\right),
\end{align}
denoting the Gaussian primitive being projected to the 2D space via the affine transformation $\Pi$\footnote{\label{fn:Pi}Suppose affine transformation $\Pi(x) = \Mat{P} \Mat{x} + \Mat{b}$ for some $\Mat{P} \in \real^{2 \times 3}$ and $\Mat{b} \in \real^2$, then when applied to matrix $\Mat{\Sigma} \in \real^{3 \times 3}$, $\Pi(\Mat{\Sigma}) = \Mat{P} \Mat{\Sigma} \Mat{P}^{\top}$.}.

\vspace{-1em}
\paragraph{Optimization.}
To acquire parameters $\Mat{\theta}$ to represent 3D scenes, \citet{kerbl20233d} employs a point cloud computed from Structure-from-Motion (SfM) software \cite{schonberger2016structure} with input images as the initial positions of Gaussian primitives.
Afterward, 3DGS optimizes parameters $\Mat{\theta}$ by minimizing the photometric error between captured images and images rendered from $\Mat{\theta}$ at the same viewpoint.
Suppose we have error function $\ell(\cdot, \cdot)$, the total photometric loss can be written as:
\begin{align} \label{eqn:loss}
\cL(\Mat{\theta}) = \mean_{(\Pi, \Mat{x}) \sim \Set{D}} \left[ \ell \left(\Mat{C}_{\Pi}(\Mat{x}; \Mat{\theta}), \widehat{\Mat{C}}_{\Pi}(\Mat{x})\right) \right],
\end{align}
where $\Set{D}$ denotes a sampling distribution of all collected pixels (and corresponding camera poses), and $\widehat{\Mat{C}}_{\Pi}(\Mat{x})$ is the corresponding ground-truth pixel color. 
In particular, 3DGS adopts $\ell_1$ distance as the loss function $\ell(\cdot, \cdot)$\footnote{Without loss of generality, we ignore the SSIM term for simplicity.}.
We note that this photometric loss can be end-to-end viewed as a functional, which maps a set of basis functions $\{\sigma_{\Pi}(\Mat{x}; \Mat{\theta}^{(i)})\}$ to a loss value according to the specified parameters $(\Pi, \Mat{x})$.

\paragraph{Adaptive Density Control.}
In conjunction with the standard gradient-based method for optimizing Eq. \ref{eqn:loss}, 3DGS incorporates an \textit{Adaptive Density Control (ADC)} procedure, which dynamically prunes invisible points and introduces new primitives using geometric heuristics to more effectively cover the scene.
We summarize this densification process here as it is of this paper's main interest.
\textit{(i)} ADC first identifies a set of points with a large expected gradient norm on the view space: $\Set{I} = \{i \in [n]: \mean [\lVert\nabla_{\Pi(\Mat{p}^{(i)})}\cL\rVert_2] \ge \epsilon_{adc}\}$. 
\textit{(ii)} For points in $\Set{I}$ with small scales $\Set{I}_{clone} = \{i \in \Set{I}: \lVert\Mat{\Sigma}^{(i)}\rVert_2 \le \tau_{adc}\}$, ADC considers them as \textit{under-reconstruction} cases. ADC yields \textit{two} off-spring primitives for each point in $\Set{I}_{clone}$ by duplicating their attributes and adjusting the positions of new offspring along the gradient direction. 
\textit{(iii)} For those who have large size $\Set{I}_{split} = \{i \in \Set{I}: \lVert\Mat{\Sigma}^{(i)}\rVert_2 \ge \tau_{adc}\}$, ADC regards them as \textit{over-reconstruction} cases. Similarly, 3DGS generates \textit{two} new off-spring points for each point in $\Set{I}_{split}$, while placing them at a random location drawn from the parent density function, and downsizing the scales by a factor of 0.8. 

\section{Methodology}

Despite the empirical success of the conventional ADC in 3DGS, it often produces redundant points given its heuristic and somewhat artisanal splitting criteria.
In this section, we establish a theoretical framework to investigate densification mechanism.
By this means, we affirmatively answer three key questions:
\textit{(i)} What is the necessary condition for a Gaussian primitive to be split? (Sec. \ref{sec:when_split})
\textit{(ii)} Where should new Gaussian off-springs be placed? (Sec. \ref{sec:optim_split})
and \textit{(iii)} How to adjust opacity for new Gaussian points? (Sec. \ref{sec:optim_split})
Combining these results, we propose a new density control strategy with hardware-efficient implementation (Sec. \ref{sec:impl}).

\subsection{Problem Setup}
\label{sec:setup}

Suppose the scene has been represented by $n$ Gaussian primitives with parameters collectively as $\Mat{\theta} = \{ \Mat{\theta}^{(i)} \}_{i=1}^{n}$ and the total loss is $\cL(\Mat{\theta})$ as defined in Eq. \ref{eqn:loss}.
Our goal is to split each Gaussian into $m_i \in \nat_{+}$ off-springs.
We denote the parameters of the $i$-th Gaussians' off-springs as $\Mat{\vartheta}^{(i)} = \{\Mat{\vartheta}^{(i)}_j\}_{j=1}^{m_i}$, where $\Mat{\vartheta}^{(i)}_j$ is the $j$-th off-spring.
Further on, we assign each Gaussian a group of coefficients $\Mat{w}^{(i)} = \{ w^{(i)}_j \in \real_{+} \}_{j=1}^{m_i}$ to reweigh their opacity.
Since the scene are approximated via summation, the process of splitting can be equivalently viewed as replacing each extant Gaussian locally with a combination of its off-springs: $\sum_{j=1}^{m_i} w^{(i)}_j \sigma(\cdot; \Mat{\vartheta}^{(i)}_j)$.
The 2D projections of each old Gaussian then become $\sum_{j=1}^{m_i} w^{(i)}_j \sigma_{\Pi}(\cdot; \Mat{\vartheta}^{(i)}_j)$ for every camera pose $\Pi$.
Note that $m_i = 1$ implies no splitting happens and the original Gaussian remains unaltered.

We collect parameters of all the new Gaussians as $\Mat{\vartheta} = \{\Mat{\vartheta}^{(i)}\}_{i=1}^{n}$, and reweighting coefficients as $\Mat{w} = \{ \Mat{w}^{(i)} \}_{i=1}^{n}$, for shorthand.
A densification algorithm determines values for $\{m_i\}_{i=1}^n$, $\Mat{w}$, and $\Mat{\vartheta}$ at each step.
After specifying values of $\Mat{w}$, these coefficients will be absorbed into the opacity values of the off-springs. 

The original ADC can be interpreted through the lens of this framework:
\textit{(i)} $m_i = 2$ if $\mean [\lVert\nabla_{\Pi(\Mat{p}^{(i)})}\cL\rVert_2] \ge \epsilon_{adc}$, or $m_i = 1$ otherwise.
\textit{(ii)} If $i \in \Set{I}_{clone}$, then $\Mat{p}^{(i)}_j - \Mat{p}^{(i)} \propto \nabla_{\Mat{p}^{(i)}}\cL$, otherwise $\Mat{p}^{(i)}_j \sim \gauss(\Mat{p}^{(i)}, \Mat{\Sigma}^{(i)})$, $\Mat{\Sigma}^{(i)}_j = 0.64 \Mat{\Sigma}^{(i)}$, for every $j \in [m_i]$.
\textit{(iii)} $w^{(i)}_j = 1, \forall j \in [m_i]$ for either case.

\subsection{When is Densification Helpful?}
\label{sec:when_split}

\begin{figure}
\centering
\includegraphics[trim=0 2 65 5,clip,width=0.9\linewidth]{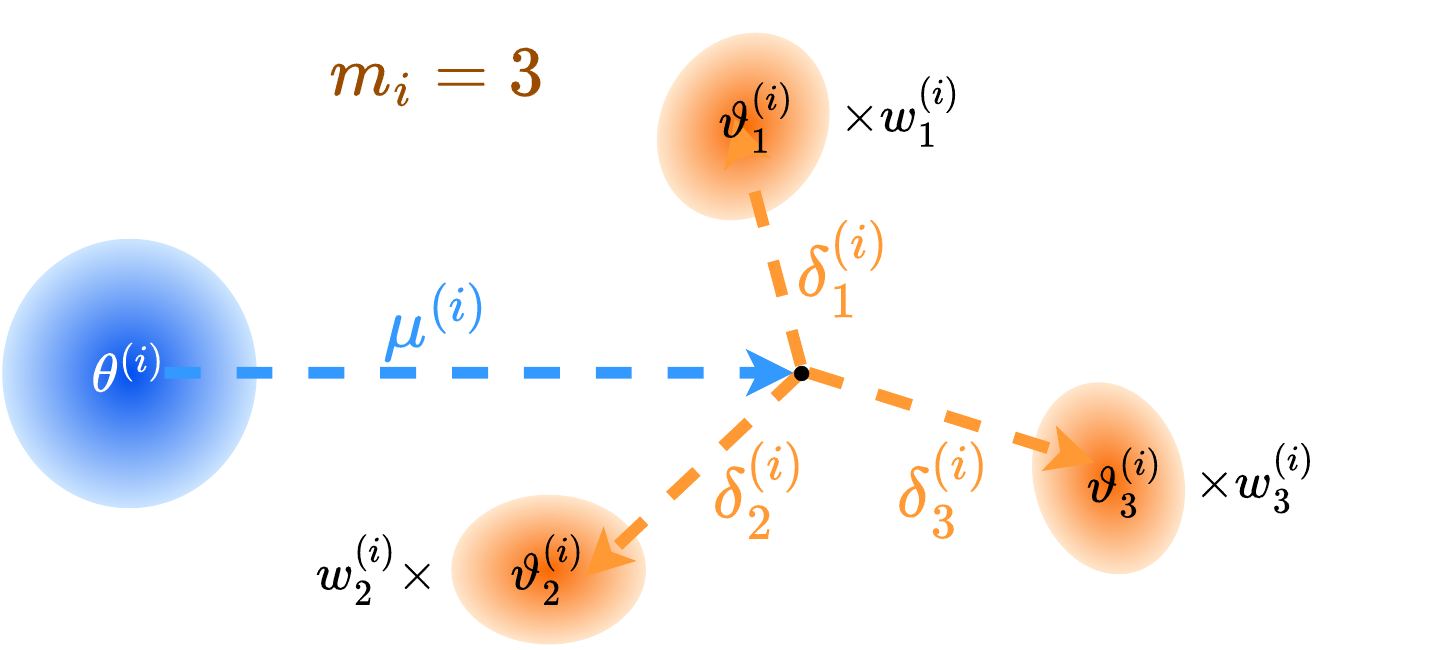}
\caption{\textbf{Illustrative notation} for the splitting process. The updates $\Mat{\vartheta}^{(i)}_j - \Mat{\theta}^{(i)}$ can be decomposed as first taking a mean-field shift $\Mat{\mu}^{(i)}$ and then applying individual updates $\Mat{\delta}^{(i)}_j$. By this decomposition, $\sum w^{(i)}_j \Mat{\delta}^{(i)}_j = \Mat{0}$.} 
\label{fig:notation}
\vspace{-1em}
\end{figure}

In this section, we theoretically elucidate the effect of the density control algorithms by examining the photometric loss after splitting.
With newly added Gaussian primitives, the loss can be evaluated as:
\begin{align} \label{eqn:aug_loss}
\cL(\Mat{\vartheta}, \Mat{w}) = \mean_{(\Pi, \Mat{x}) \sim \Set{D}} \left[ \ell \left(\Mat{C}_{\Pi}(\Mat{x}; \Mat{\vartheta}, \Mat{w}), \widehat{\Mat{C}}_{\Pi}(\Mat{x})\right) \right],
\end{align}
where $\Mat{C}_{\Pi}(\Mat{x}; \Mat{\vartheta}, \Mat{w})$ denotes the color rendered for pixel $\Mat{x}$ via Eq. \ref{eqn:alpha_blend} with new parameters $\Mat{\vartheta}$ and $\Mat{w}$.
The original ADC \cite{kerbl20233d} has no guarantee that the loss will decrease after splitting.
As we show later, densifying a random point can even increase the loss.
\citet{bulo2024revising} and \citet{kheradmand20243d} adjust opacity values for Gaussian off-springs to preserve total densities locally after splitting.
In this setup, the loss $\cL(\Mat{\vartheta}, \Mat{w})$ remains approximately equivalent to $\cL(\Mat{\theta})$.
However, as we will demonstrate next, this preservation principle is not necessarily the most effective approach.
By choosing appropriate Gaussian primitives, it is even possible to further reduce the loss.

From this point forward, we consider two additional practical conditions.
First, to ensure the total opacity is conservative after splitting, we impose the constraint that $\sum_{j=1}^{m_i} w^{(i)}_j = 1$ for every $i \in [n]$.
Second, we assume the parameters of off-springs are close to the original parameters, \ie $\{\Mat{\vartheta}_j^{(i)}\}_{j=1}^{m_i}$ are within a neighborhood of $\Mat{\theta}^{(i)}$: $\lVert  \Mat{\vartheta}_j^{(i)} - \Mat{\theta}^{(i)} \rVert_2 \le \epsilon$ for some $\epsilon > 0$.
Let us define $\Mat{\mu}^{(i)} = \sum_{j=1}^{m_i} w^{(i)}_j \Mat{\vartheta}^{(i)}_j - \Mat{\theta}^{(i)}$ as the average displacement of the $i$-th Gaussian after densification, and $\Mat{\delta}^{(i)}_j = (\Mat{\vartheta}^{(i)}_j - \Mat{\theta}^{(i)}) - \Mat{\mu}^{(i)}$ denotes an offset additional to $\Mat{\mu}^{(i)}$ for each off-spring.
We visualize the splitting process in Fig. \ref{fig:notation} for a better illustration of our notations.
The splitting process consists of two steps: first translating the parent Gaussian by an offset $\Mat{\mu}^{(i)}$, and second, generating offspring with individual shifts $\{\Mat{\delta}^{(i)}_j\}_{j \in [m_i]}$ whose mean is zero.

With all these settings, below we present our first main result, which decomposes the loss after splitting:
\begin{theorem} \label{thm:sec_ord_approx} 
Assume $\cL(\Mat{\vartheta}, \Mat{w})$ has bounded third-order derivatives with respect to $\Mat{\vartheta}$, then
\begin{align*}
\cL(\Mat{\vartheta}, \Mat{w}) &= \highlighteqn{azure!10}{\cL(\Mat{\theta}) + \nabla_{\Mat{\theta}} \cL(\Mat{\theta})^\top \Mat{\mu} + \frac{1}{2} \Mat{\mu}^\top \nabla_{\Mat{\theta}}^2 \cL(\Mat{\theta}) \Mat{\mu}} \\ 
&+ \highlighteqn{orange!10}{\frac{1}{2} \sum_{i=1}^{n} \sum_{j=1}^{m_i} w^{(i)}_j \Mat{\delta}^{(i)\top}_j \Mat{S}^{(i)}(\Mat{\theta}) \Mat{\delta}^{(i)}_j} + \bigO(\epsilon^3),
\end{align*}
where $\Mat{\mu} = \begin{bmatrix} \Mat{\mu}^{(1)\top} & \cdots & \Mat{\mu}^{(n)\top} \end{bmatrix}^{\top}$ concatenates all average offsets and $\Mat{S}^{(i)}(\Mat{\theta}) \in \real^{\dim\Theta \times \dim\Theta}$ is defined as:
\begin{align}
\Mat{S}^{(i)}(\Mat{\theta}) = \mean_{(\Pi, \Mat{x}) \sim \Set{D}} \left[ \frac{\partial \ell}{\partial \sigma_{\Pi}(\Mat{x}; \Mat{\theta}^{(i)})} \nabla^2_{\Mat{\theta}^{(i)}} \sigma_{\Pi}(\Mat{x}; \Mat{\theta}^{(i)}) \right] \nonumber.
\end{align}
\end{theorem}
\noindent All proofs are provided in the supplementary material, Appendix~\ref{sec:proofs}.
Theorem~\ref{thm:sec_ord_approx} decouples and groups the effects of the mean displacement and the individual offsets of each offspring.
The first three terms, highlighted in \colorbox{azure!10}{blue}, are referred to as the \textit{mean shift terms}, which collectively represent the impact of shifting the overall mean.
Since shifting all off-springs simultaneously is equivalent to applying the same offset to the original Gaussian, the term involving individual offsets, highlighted in \colorbox{orange!10}{orange}, fully captures the intrinsic effect of the splitting process.
We isolate each term within this summation and refer to it as a \textit{splitting characteristic function}, which fully describes the effect of splitting on a single Gaussian point:
\begin{align}
\label{eqn:Delta}
\Delta^{(i)}(\Mat{\delta}^{(i)}, \Mat{w}^{(i)}; \Mat{\theta}) \triangleq \frac{1}{2} \sum_{j=1}^{m_i} w^{(i)}_j \Mat{\delta}^{(i)\top}_j \Mat{S}^{(i)}(\Mat{\theta}) \Mat{\delta}^{(i)}_j.
\end{align}
Essentially, the splitting characteristic function takes a quadratic form with respect to the matrix $\Mat{S}^{(i)}(\Mat{\theta})$.
We refer to $\Mat{S}^{(i)}(\Mat{\theta})$ as the \textit{splitting matrix}, which fully governs the behavior of the splitting characteristic function.
Theorem~\ref{thm:sec_ord_approx} draws two insights as below:

\vspace{-0.5em}
\paragraph{Densification escapes saddle points.}
By using the RHS of Theorem~\ref{thm:sec_ord_approx} as a surrogate for minimizing the loss, one can observe that optimizing the mean shift terms does not require densification but can be achieved by standard gradient descent \footnote{More precisely, this aligns with the Newton method, where gradients are preconditioned by the inverse Hessian.}.
However, it is worth noting that the loss function in Eq. \ref{eqn:loss} is a highly non-convex objective, exhibiting numerous superfluous saddle points in its loss landscape (see the right of Fig. \ref{fig:teaser} for an illustration).
Gradients are prone to getting trapped at these saddle points, at which point gradient descent ceases to yield further improvements.
However, by densifying Gaussian points at saddle points into multiple particles, a new term -- captured by the splitting characteristic functions $\Delta^{(i)}$ -- emerges.
These terms can become negative to further reduce the loss.
\textit{Thus, densification serves as an effective mechanism for escaping saddle points.}

\vspace{-0.5em}
\paragraph{When does splitting decrease loss?}
In fact, a finer-grained analysis can be done with the splitting characteristic function $\Delta^{(i)}$.
Since the overall contribution of splitting to the loss is expressed as a sum of quadratic functions over the splitting matrices, $\Delta^{(i)}$ can be negative to decrease the loss only if the \textit{associated splitting matrix is not positive semi-definite}.
This implies that a necessary condition for performing a split is $\lambda_{min}(\Mat{S}^{(i)}(\Mat{\theta})) < 0$, where $\lambda_{min}(\cdot)$ denotes the smallest eigenvalue of the specified matrix.

\subsection{Optimal Density Control}
\label{sec:optim_split}

\begin{figure}[tb]
    \centering
    \includegraphics[trim= 20 10 3 5,clip,width=\linewidth]{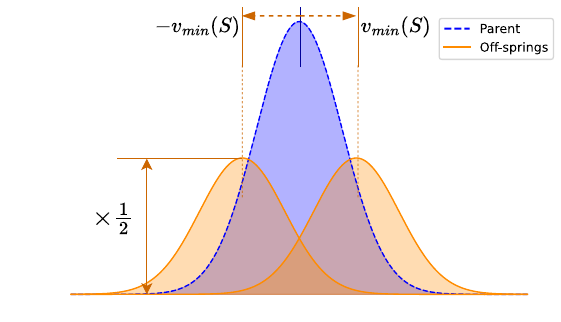}
    \vspace{-2em}
    \caption{\textbf{Illustration of Steepest Density Control}. SDC, as the optimal solution to Eq. \ref{eqn:optim_goal}, takes the steepest descent on the loss after splitting. Geometrically, it moves two off-spring Gaussians to opposite directions along the smallest eigenvector of the splitting matrix and shrinks the opacity of each Gaussian by 0.5.}
    \label{fig:optim_split}
    \vspace{-1em}
\end{figure}

Based on Theorem \ref{thm:sec_ord_approx}, we can derive even stronger results.
We intend to find a density control strategy that introduces a minimal number of points while achieving the steepest descent in the loss.
Maximizing loss descent at each step enforces low reconstruction errors and accelerates convergence.
To this end, we formulate the following constrained optimization objective:
\begin{align}
\label{eqn:optim_goal}
& \min \cL(\Mat{\vartheta}, \Mat{w}),
\text{s.t.} \left\lVert \Mat{\vartheta}^{(i)}_{j} - \Mat{\theta}^{(i)} \right\rVert_2 \le \epsilon, \sum_{j=1}^{m_i} w^{(i)}_{j} = 1,
\end{align}
which seeks an optimal configuration of the number of off-springs $m_i \in \nat_{+}$, reweighting coefficients $w^{(i)}_{j} \in \real_{+}$, and updates to the parameters for new primitives $\{\Mat{\delta}^{(i)}_{j}\}$ for all $i \in [n], j \in [m_i]$ to maximize loss descent.

\vspace{1em}
While solving Eq.~\ref{eqn:optim_goal} directly is difficult, Theorem \ref{thm:sec_ord_approx} provides an ideal second-order approximation as a surrogate.
Since the effect of splitting Gaussians is fully characterized by the splitting characteristic function, it is sufficient to consider minimizing $\Delta^{(i)}$ with respect to $\{\Mat{\delta}^{(i)}_j\}$. 
Consequently, the solution is simple and analytical.
Our main result is presented below:

\begin{theorem} \label{thm:optim_split}
The optimal solution to Eq. \ref{eqn:optim_goal} has two folds:
\begin{enumerate}
\item If splitting matrix $\Mat{S}^{(i)}(\Mat{\theta})$ is positive semi-definite $\lambda_{min}(\Mat{S}^{(i)}(\Mat{\theta})) \ge 0$, then splitting cannot decrease the loss.
In this case, we set $m_i = 1$ and no splitting happens.
\item Otherwise, if $\lambda_{min}(\Mat{S}^{(i)}(\Mat{\theta})) < 0$, then the following splitting strategy  minimizes $\Delta^{(i)}(\Mat{\delta}^{(i)}, \Mat{w}^{(i)}; \Mat{\theta})$ subject to $\lVert\Mat{\delta}^{(i)}_j \rVert \le 1$ for every $j \in [m_i]$:
\begin{align*}
\begin{array}{cc}
m^*_i = 2, & w^{(i)*}_1 = w^{(i)*}_2 = \frac{1}{2}, \\
\\
\Mat{\delta}^{(i)*}_1 = \Mat{v}_{min}(\Mat{S}^{(i)}(\Mat{\theta})), & \Mat{\delta}^{(i)*}_2 = -\Mat{v}_{min}(\Mat{S}^{(i)}(\Mat{\theta})),
\end{array}
\end{align*}
where $\Mat{v}_{min}(\cdot)$ denotes the eigenvector associated with the smallest eigenvalue $\lambda_{min}(\cdot)$.
\end{enumerate}
\end{theorem}
\noindent We term the splitting strategy outlined in Theorem \ref{thm:optim_split} as \textit{Steepest Density Control (SDC)}.
It has four practical implications:
\begin{tcolorbox}[
    colback=orange!10,
    colframe=orange!10,
    boxrule=0pt,
    arc=2mm,
    left=0mm, right=2mm, top=0.5em, bottom=0.5em,
    enhanced,
    sharp corners
]
\begin{enumerate}[label={\roman*}),itemsep=0.25em,parsep=0.em,topsep=1em]
\item A necessary condition for densification is the \textit{positive indefiniteness} of the splitting matrices.
\item Splitting each Gaussian into \textit{two} off-springs is sufficient, and generating more off-springs offers no additional benefit.
\item The magnitudes of the new primitives must be downscaled by \textit{a factor of two} to preserve local opacity.
\item The parameters of the two off-springs should be updated along the positive/negative directions of the \textit{eigenvector corresponding to the least eigenvalue of the splitting matrices}.
\end{enumerate}
\end{tcolorbox}
Further on, the result in \textit{(i)}, as already shown in Sec. \ref{sec:when_split}, can be employed as a filter to reduce the number of points to be densified.
Implication by \textit{(ii)} justifies the common choice of ``two off-springs'' in existing ADC algorithms \cite{kerbl20233d} and ensures a moderate growth rate.
The result \textit{(iii)} contrasts with the approaches of \citet{bulo2024revising} and \citet{kheradmand20243d}, where opacities of new Gaussians are adjusted based on rendering-specific schemes.
Notably, both opacity adjustment methods in these works are inexact and do not necessarily preserve total opacity.
The point \textit{(iv)} gives two update directions that shift points to the non-saddle area when they are optimized to the stationary point, as illustrated on the right of Fig. \ref{fig:teaser}.
We further demonstrate this splitting scheme in Fig. \ref{fig:optim_split}, which visualizes how our splitting strategy locally refines the geometry.

\subsection{Steepest Gaussian Splatting}
\label{sec:impl}

In this section, we instantiate a 3DGS optimization algorithm with SDC as the density control scheme, dubbed \textbf{SteepGS}.
At the core of SteepGS is to compute the splitting matrices and leverage them to decide when and how to split points.
Algorithm \ref{alg:main} in Appendix \ref{sec:pseudocode} provides a reference implementation for SteepGS.
We show that these can be efficiently implemented and integrated into the CUDA kernel.

Revisiting the form of splitting matrix in Theorem \ref{thm:sec_ord_approx}, we note that $\partial \ell / \partial \sigma_{\Pi}(\Mat{x}; \Mat{\theta}^{(i)})$ denotes the gradient of (scalar) loss $\ell$ back-propagated to the output of the $i$-th Gaussian primitive, and $\nabla^2_{\Mat{\theta}} \sigma(\Mat{\theta}^{(i)}, \Mat{x})$ is the Hessian matrix of the $i$-th Gaussian with respect to its own parameters $\Mat{\theta}^{(i)}$.
It is noteworthy that splitting matrices are defined per point and only rely on the Hessian of each Gaussian individually.
Therefore, the total memory footprint to store the splitting matrices is $\bigO(n(\dim \Theta)^2)$.
When we only consider the mean positions $\{\Mat{p}^{(i)}\}_{i=1}^{n}$ as the parameters, $\dim\Theta = 3$.

To compute splitting matrix, we note that the gradient $\partial \ell/\partial \sigma_{\Pi}(\Mat{x}; \Mat{\theta}^{(i)})$ has already been acquired during the back-propagation when computing $\nabla_{\Mat{\theta}} \cL(\Mat{\theta}, \Mat{x})$.
The remaining part, the Hessian matrix $\nabla^2_{\Mat{\theta}} \sigma_{\Pi}(\Mat{x}; \Mat{\theta}^{(i)})$ can be approximated analytically as (see derivations in Appendix \ref{sec:hessian}):
\begin{align*}
&\nabla^2_{\Mat{\theta}} \sigma_{\Pi}(\Mat{x}; \Mat{\theta}^{(i)}) \approx \sigma^{(i)} \Mat{\Upsilon} \Mat{\Upsilon}^\top - \sigma^{(i)} \Mat{P}^\top \Pi(\Mat{\Sigma}^{(i)})^{-1}\Mat{P},
\end{align*}
where $\Mat{\Upsilon} \triangleq \Mat{P}^\top \Pi(\Mat{\Sigma}^{(i)})^{-1}(\Mat{x} - \Pi(\Mat{p}^{(i)}))$, $\Mat{P}$ denotes the projection matrix given by $\Pi$\footnotemark[\getrefnumber{fn:Pi}], and $\sigma^{(i)} \triangleq \sigma_{\Pi}(\Mat{x}; \Mat{\theta}^{(i)})$.
The $\Mat{S}^{(i)}$ can be computed in parallel, and intermediate results such as $\sigma, \Pi(\Mat{\mu}), \Pi(\Mat{\Sigma})$ can be reused from previous forward computation.
The complexity of computing the minimum eigenvalue and eigenvector for $\dim \Theta \times \dim \Theta$ matrices is $\bigO((\dim \Theta)^2)$.
As we only use the position parameter, the least eigenvalue and eigenvector for $3 \times 3$ matrices can be calculated using the root formula \cite{smith1961eigenvalues}.
Although splitting matrices leverage the second-order information, it is distinct from the full Hessian of the total loss $\nabla^2_{\Mat{\theta}} \cL$ (\eg used in \citet{hanson2024pup})
by precluding cross terms in full Hessian $\nabla^2_{\Mat{\theta}} \cL$ \cite{wu2019splitting}.
Surprisingly, the structure of the problem allows for the pointwise identification of saddle points with only partial Hessian information provided by splitting matrices.

\section{Experiments}

\begin{figure*}[tbh]
  \centering
  \includegraphics[width=0.98\textwidth]{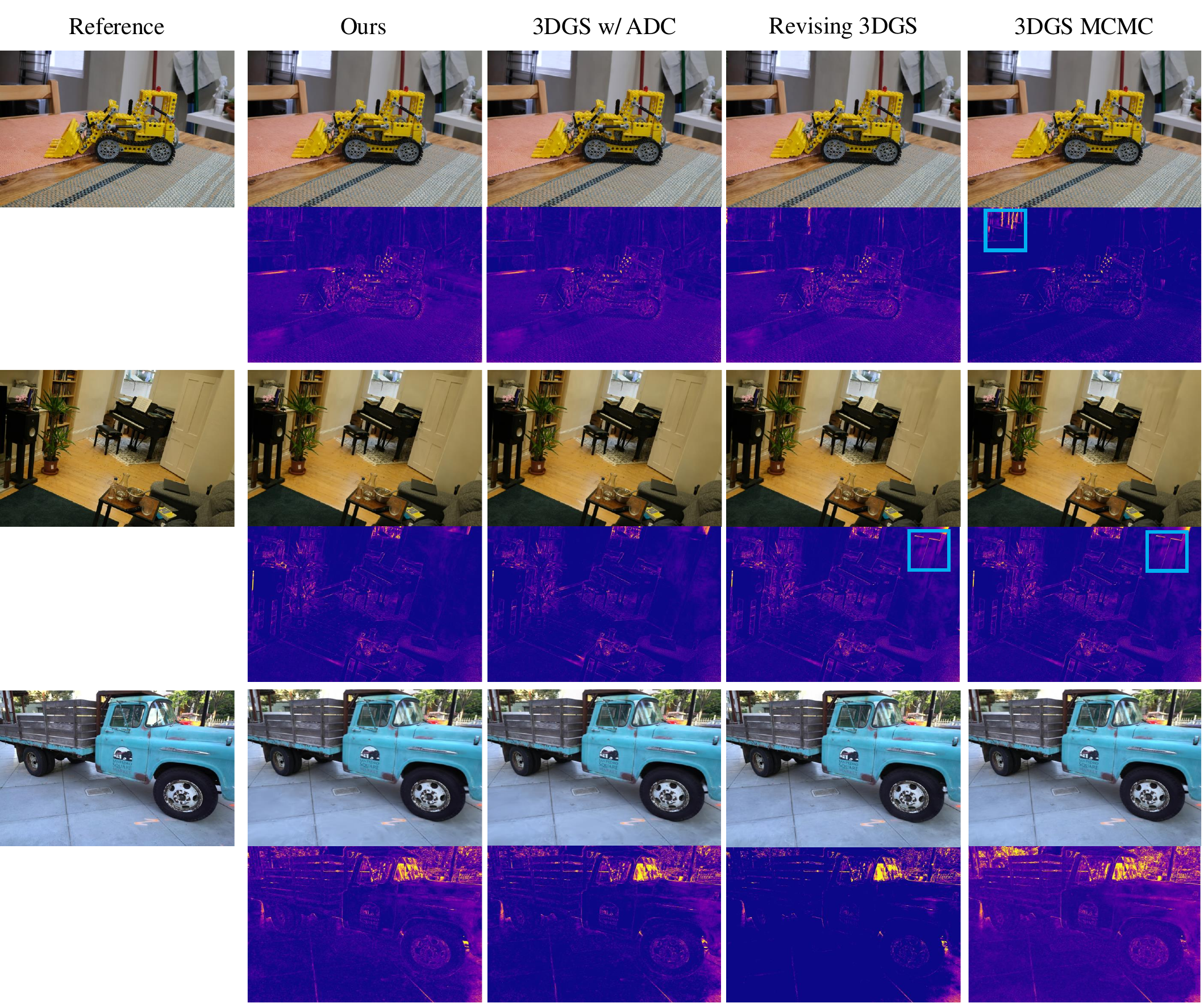}
  \caption{\textbf{Qualitative Results.} We compare our \textbf{SteepGS} with other densification baselines. For each scene, the first row shows the rendered view, while the second row visualizes the error with respect to the ground truth. Key details are highlighted in the \textcolor{azure!75}{blue} box. }
  \vspace{-1em}
  \label{fig:qualitative}
\end{figure*}
In this section, we empirically validate the effectiveness of our proposed SteepGS.

\subsection{Settings}

In our main experiments, we train 3DGS with various densification schemes and compare their final number of points and the novel view synthesis quality. 

\vspace{-1em}
\paragraph{Datasets.}
We compare our methods and other baselines on three challenging real-world datasets: Mip-NeRF 360 \cite{barron2022mip} including three outdoor and four indoor scenes, Tanks \& Temples \cite{knapitsch2017tanks} with two outdoor scenes, and Deep Blending \cite{hedman2018deep} containing two indoor scenes.

\vspace{-1.5em}
\paragraph{Baselines.}
We compare SteepGS with the original ADC and the other two densification schemes:
3DGS-MCMC \cite{kheradmand20243d} and Revising-GS \cite{bulo2024revising}.
We use the official codebases for 3DGS and 3DGS-MCMC, while re-implementing Revising-GS based on the codebase of the original 3DGS. We added a 3DGS-Thres. baseline which is modified from the original 3DGS ADC such that the training stops when meeting the same amount of points of other compared methods.
Both 3DGS-MCMC and Revising-GS require a maximum limit of the Gaussian number.
We choose this number as the number of Gaussion our SteepGS yields.
For all these baselines, we use their default hyper-parameters.

\vspace{-1em}
\paragraph{Implementation Details.}

Our SteepGS is implemented based on the codebase of the official 3DGS codebase \cite{kerbl20233d}.
We further customize the CUDA kernel to compute the splitting matrices and their eigen-decompositions. 
We follow the standard training pipeline of 3DGS and perform density control for every 100 steps starting from the 500th step.
The threshold for the smallest eigenvalues of splitting matrices is chosen as $-1e-6$. 
All other hyper-parameters are kept the same with 3DGS's default settings. Each of our per-scene training are conducted on a single NVIDIA V100 GPU.

\begin{table}[t!]
    \centering
    \renewcommand{\arraystretch}{1.5}
    \resizebox{\linewidth}{!}{
    \begin{tabular}{lcccc}
        \toprule
        & \multicolumn{4}{c}{MipNeRF360} \\
        & \# Points $\downarrow$ & PSNR $\uparrow$ & SSIM $\uparrow$ & LPIPS $\downarrow$ \\ \hline
        3DGS \cite{kerbl20233d} & 3.339 & 29.037 & 0.872 & 0.183 \\
        \hline
        3DGS + Thres. \cite{kerbl20233d} & 1.632 & 27.851 & 0.848 & 0.227 \\
        3DGS-MCMC \cite{kheradmand20243d} & 1.606 & \cellsecond{28.149} & \cellbest{0.853} & \cellbest{0.204} \\
        Revising 3DGS \cite{bulo2024revising} & 1.606 & 28.085 & \cellsecond{0.850} & \cellsecond{0.212} \\
        \hline
        SteepGS (Ours) & 1.606 & \cellbest{28.734} & 0.857 & 0.211 \\
    \end{tabular}}

    \resizebox{\linewidth}{!}{
    \begin{tabular}{lcccc}
        \hline
        & \multicolumn{4}{c}{Tank \& Temple} \\
        & \# Points $\downarrow$ & PSNR $\uparrow$ & SSIM $\uparrow$ & LPIPS $\downarrow$ \\ \hline
        3DGS \cite{kerbl20233d} & 1.830 & 23.743 & 0.848 & 0.177 \\
        \hline
        3DGS + Thres. \cite{kerbl20233d} & 0.973 & 22.415 & 0.812 & 0.218 \\
        3DGS-MCMC \cite{kheradmand20243d} & 0.957 & \cellsecond{22.545} & \cellsecond{0.817} & \cellsecond{0.204} \\
        Revising 3DGS \cite{bulo2024revising} & 0.957 & 22.339 & 0.811 & 0.216 \\
        \hline
        SteepGS (Ours) & 0.958 & \cellbest{23.684} & \cellbest{0.840} & \cellbest{0.194} \\
    \end{tabular}}

    \resizebox{\linewidth}{!}{
    \begin{tabular}{lcccc}
        \hline
        & \multicolumn{4}{c}{Deep Blending} \\
        & \# Points $\downarrow$ & PSNR $\uparrow$ & SSIM $\uparrow$ & LPIPS $\downarrow$ \\ \hline
        3DGS \cite{kerbl20233d} & 2.818 & 29.690 & 0.904 & 0.244 \\
        \hline
        3DGS + Thres. \cite{kerbl20233d} & 1.326 & 29.374 & 0.900 & 0.250 \\
        3DGS-MCMC \cite{kheradmand20243d} & 1.296 & \cellsecond{29.439} & \cellsecond{0.901} & \cellbest{0.237} \\
        Revising 3DGS \cite{bulo2024revising} & 1.296 & 29.439 & 0.895 & \cellsecond{0.245} \\
        \hline
        SteepGS (Ours) & 1.296 & \cellbest{29.963} & \cellbest{0.905} & 0.250 \\ \bottomrule
    \end{tabular}}

    \caption{Comparison with representative baselines, including 3DGS \cite{kerbl20233d}, 3DGS + Thres. \cite{kerbl20233d}, and Revising 3DGS \cite{bulo2024revising}. The unit of ``\# Points'' is million. The \colorbox{orange!35}{best} and \colorbox{orange!10}{second best} approaches for improving densification are marked in colors.}
    \label{tab:quan_res}
    \vspace{-2em}
\end{table}

\subsection{Results}
We adopt PSNR, SSIM \cite{wang2004image}, and LPIPS \cite{zhang2018perceptual} as evaluation metrics for view synthesis quality, and present the final number of yielded points and training time as the efficiency metrics.
Tab. \ref{tab:quan_res} presents our quantitative experiments.
Compared to the original adaptive density control, our method achieves comparable rendering performance on the Mip-NeRF 360 dataset while reducing the number of Gaussians by approximately $50\%$. Our method outperforms 3DGS-Thres., 3DGS-MCMC and Revising 3DGS baselines when the two methods are trained to generate the same number of Gaussians. Similarly, on the Tank \& Temple and DeepBlending datasets, our approach attains around a $50\%$ reduction in the number of Gaussians and even achieves higher PSNR values for DeepBlending scenes. Fig.~\ref{fig:qualitative} presents qualitative comparisons between our densification algorithm and baseline methods. As illustrated by the difference maps, although our method generates only half the number of new points compared to the original densification approach, it maintains competitive rendering quality and preserves many details.
These results suggest that the additional Gaussians produced by the original adaptive density control are redundant. In contrast, our method mitigates redundancy during the training process through the densification procedure. Notably, unlike compression-based methods such as LightGaussian \cite{fan2023lightgaussian}, which identify and prune less important Gaussians after the entire training process, our method achieves these savings without the need for post-training steps.

\begin{figure}[t]
  \centering
  \includegraphics[width=0.45\textwidth]{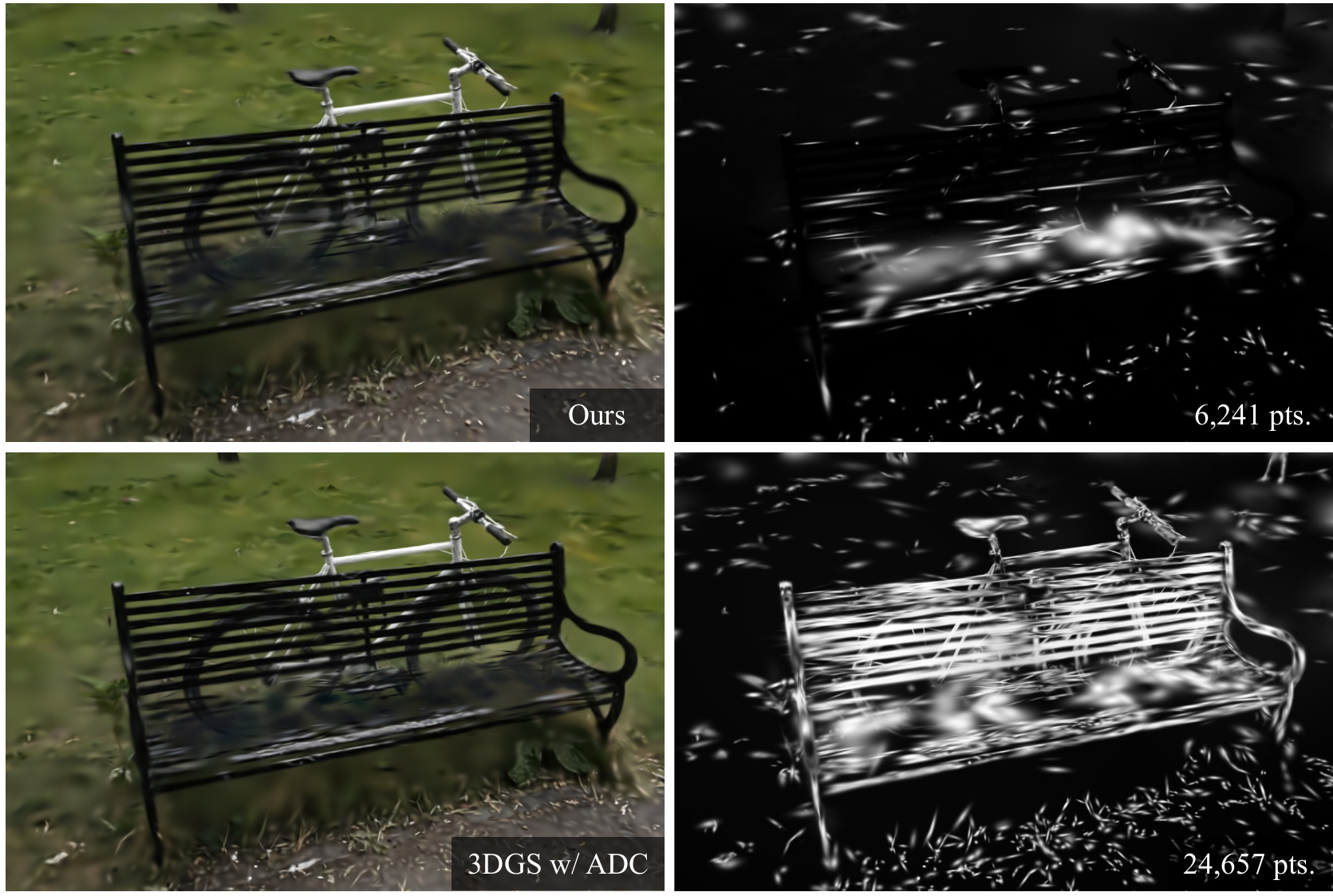}
  \caption{\textbf{Visualization of splitting points.} The rendered views are present on the left and the corresponding points to be split are visualized on the right.}
  \label{fig:vis_densify}
  \vspace{-1em}
\end{figure}

\subsection{Visualization and Interpretation}

In Fig. \ref{fig:vis_densify}, we visualize the points filtered by our strategy and the original splitting strategy at the 1,000th iteration, respectively. The rendered RGB images show that the backrest of the bench has been trained to capture its basic shape and appearance, whereas the seat still lacks clear details. As depicted in the figure, our method effectively concentrates Gaussian splitting on the seat, leaving other regions to the optimizer. In contrast, the original adaptive density control allocates splitting to the backrest, resulting in approximately four times as many splitting points as our method. This suggests that our approach enhances efficiency by focusing splitting on areas that require more detail, thereby improving overall rendering performance in a more efficient way.

\section{Conclusion}
This work addresses the inefficiencies in 3D Gaussian Splatting (3DGS), a leading technique for real-time, high-resolution novel view synthesis. While effective, its densification process often generates redundant points, leading to high memory usage, slower performance, and increased storage demands--hindering deployment on resource-constrained devices.
To tackle this, we introduced a theoretical framework that clarifies and optimizes density control in 3DGS. Our analysis highlights the necessity of splitting for escaping saddle points and establishes optimal conditions for densification, including the minimal number of offspring Gaussians and their parameter updates.
Building on these insights, we proposed \textbf{SteepGS}, which integrates \textit{steepest density control} to maintain compact point clouds. SteepGS reduces Gaussian points by 50\% without sacrificing rendering quality, improving efficiency and scalability for practical use.

\subsubsection*{Acknowledgments}
PW thanks Forrest Iandola, Zhen Wang, Jonathon Luiten, Nikolaos Sarafianos, and Amit Kumar for helpful discussion during the preparation of this work.
Any statements, opinions, findings, and conclusions or recommendations expressed in this material are those of the authors and do not necessarily reflect the views of their employers or the supporting entities. 

{
    \small
    \bibliographystyle{ieeenat_fullname}
    \bibliography{main}
}

\clearpage
\setcounter{page}{1}
\appendix
\onecolumn

{\centering
\Large
\textbf{\thetitle}\\
\vspace{0.5em}Supplementary Material \\
\vspace{1.0em}
}

\section{Implementation Details}
\label{sec:impl_details}

\subsection{Pseudocode}
\label{sec:pseudocode}

We provide a reference pseudocode for our method, SteepGS, in Algorithm~\ref{alg:main}.
We highlight the main differences from the original ADC in \colorbox{orange!10}{orange}.
The overall procedure consists of two main components.
First, the algorithm estimates the splitting matrices on the fly in a mini-batch manner.
Second, at regular intervals, the accumulated splitting matrices are used to decide whether to split a Gaussian point and where to place the resulting offspring.
Our algorithm is designed to be general and can be integrated with other point selection criteria, such as the gradient-based strategy used in the original ADC.
Finally, we note that all \texttt{for} loops in the pseudocode are executed in parallel for efficiency.

\begin{algorithm*}[h]
\caption{Steepest Gaussian Splatting (SteepGS)}
\label{alg:main}
\begin{algorithmic}
\STATE {\bf Input}: An initial point cloud of Gaussians $\Mat{\theta} = \{(\Mat{\theta}^{(i)}, o^{(i)})\}_{i=1}^{|\Mat{\theta}|}$; A loss function $\cL(\Mat{\theta})$ associated with a training set $\Set{D}(\Set{X})$; A stepsize $\epsilon > 0$; A splitting matrix threshold $\varepsilon_{split} \le 0$; Total number of iterations $T$; Densification interval $T_{split}$.
\FOR{each training step $t = 1, \cdots, T$}
\IF{$t\mod{T_{split}} \ne 0$}
\STATE Sample a batch of data points $\Mat{x} \sim \Set{D}(\Set{X})$ and compute loss function $\cL(\Mat{\theta}, \Mat{x})$. 
\FOR{each Gaussian $i = 1, \cdots, |\Mat{\theta}|$}
\STATE Update each Gaussian parameters $\Mat{\theta}^{(i)}, o^{(i)}$ via standard gradient descent.
\STATE Accumulate gradients: $\Mat{G}^{(i)} \gets \Mat{G}^{(i)} + \nabla_{\Mat{\theta}^{(i)}} \cL(\Mat{\theta}, \Mat{x})$.
\STATE \highlightalg{orange!10}{Accumulate splitting matrix: $\Mat{S}^{(i)} \gets \Mat{S}^{(i)} + \partial_{\sigma^{(i)}} \ell(\Mat{\theta}, \Mat{x}) \nabla^2 \sigma(\Mat{\theta}^{(i)}, \Mat{x})$.}
\ENDFOR
\ELSE
\FOR{each Gaussian $i = 1, \cdots, |\Mat{\theta}|$}
\STATE Obtain average gradient and splitting matrix: $\Mat{G}^{(i)} \gets \Mat{G}^{(i)} / T_{split}$, $\Mat{S}^{(i)} \gets \Mat{S}^{(i)} / T_{split}$.
\STATE \highlightalg{orange!10}{Compute the smallest eigenvalue and the associated eigenvector for the splitting matrix: \\
$\lambda \gets \lambda_{min}(\Mat{S}^{(i)})$, $\Mat{\delta} \gets \Mat{v}_{min}(\Mat{S}^{(i)})$.}
\IF{condition on $\Mat{G}^{(i)}$ and \colorbox{orange!10}{$\lambda < \varepsilon_{split}$}}
\STATE \highlightalg{orange!10}{Replace this Gaussian with two Gaussian off-springs:\\
$\Mat{\theta} \gets \Mat{\theta} \setminus \{(\Mat{\theta}^{(i)}, o^{(i)})\} \cup \{ (\Mat{\theta}^{(i)} + \epsilon \Mat{\delta}, o^{(i)}/2), (\Mat{\theta}^{(i)} - \epsilon \Mat{\delta}, o^{(i)}/2) \}$}
\ENDIF
\ENDFOR
\ENDIF
\ENDFOR
\STATE {\bf Return} $\Mat{\theta}$
\end{algorithmic}
\end{algorithm*}

\subsection{Variants}

\paragraph{Densification with Increment Budget.}
Recent densification algorithms~\cite{bulo2024revising, kheradmand20243d} have shown that fixing the number or ratio of incremental points can lead to a more compact Gaussian point cloud.
This corresponds to imposing a global constraint on the total number of new points, $\sum_{i \in [n]} m_i \le 2K$, when solving the objective in Eq.~\ref{eqn:optim_goal}:
\begin{align}
& \min \cL(\Mat{\vartheta}, \Mat{w}),
\quad \text{s.t.} \left\lVert \Mat{\vartheta}^{(i)}_{j} - \Mat{\theta}^{(i)} \right\rVert_2 \le \epsilon, \sum_{j=1}^{m_i} w^{(i)}_{j} = 1, \sum_{i \in [n]} m_i \le 2K,
\end{align}
where $K$ is the maximum number of increased points.
According to Theorem~\ref{thm:optim_split}, the maximal loss reduction achieved by splitting the $i$-th Gaussian is given by $\Delta^{(i)*} \propto \lambda_{min}(\Mat{S}^{(i)}(\Mat{\theta}))/2$.
Therefore, the optimal point selection maximizing loss descent can be done by efficiently choosing Gaussians with least-$K$ values of $\lambda_{min}(\Mat{S}^{(i)})$ once the total number of Gaussians with negative $\lambda_{min}(\Mat{S}^{(i)})$ surpasses $K$, \ie $|\{ i \in [n] : \lambda_{min}(\Mat{S}^{(i)}) < 0 \}| > K$.

\paragraph{Compactest Splitting Strategy.}
There also exists a theoretically most compact splitting strategy.
Theorem~\ref{thm:sec_ord_approx} suggests that the optimal displacement $\Mat{\mu}$ corresponds to the standard negative gradient $\nabla \cL(\Mat{\theta})$, which yields a typical $\bigO(\epsilon)$ decrease in loss at non-stationary points.
In contrast, splitting introduces a summation of splitting characteristic functions, each governed by its associated splitting matrix, resulting in a cumulative effect of order $\bigO(\epsilon^2)$.
This theoretical insight leads to an important implication: \textit{a Gaussian should be split only when its gradient is small}.
Otherwise, splitting introduces redundant Gaussians that offer little improvement in loss.
The compactest splitting condition can be formulated as below:
\begin{align*}
\lVert \Mat{G}^{(i)} \rVert \le \varepsilon_{grad} \quad \text{and} \quad \lambda_{min}(\Mat{S}^{(i)}) < \varepsilon_{split}, \quad \forall i \in [n],
\end{align*}
where $\varepsilon_{grad} > 0$ is a chosen hyper-parameter.
While this conclusion may appear to contradict the original ADC strategy, we argue that ADC actually examines the variance of the gradient by estimating $\mean[\lVert \Mat{G}^{(i)} \rVert]$, rather than the norm of its expectation, \ie $\mean[\lVert \Mat{G}^{(i)} \rVert]$.
Thus, our condition does not contradict the original approach, but rather complements it by offering a more principled criterion.

\subsection{Eigendecomposition}

In our experiments, we only take position parameters into the consideration for steepest splitting descent.
This simplifies the eigendecomposition of splitting matrices to be restricted to symmetric $3 \times 3$ matrices.
We can follow the method by \cite{smith1961eigenvalues} to compute the eigenvalues.
The characteristic equation of a symmetric $3 \times 3$ matrix $\Mat{A}$ is:
\begin{align*}
\det(\alpha \Mat{I} - \Mat{A}) = \alpha^3 - \alpha^2 \operatorname{tr}(\Mat{A})
- \alpha \frac{1}{2} \left( \trace(\Mat{A}^2) - \trace^2(\Mat{A}) \right)
- \det(\Mat{A}) = 0.
\end{align*}
An affine change to $\Mat{A}$ will simplify the expression considerably, and lead directly to a trigonometric solution. 
If $\Mat{A} = p\Mat{B} + q\Mat{I}$, then $\Mat{A}$ and $\Mat{B}$ have the same eigenvectors, and $\beta$ is an eigenvalue of $\Mat{B}$ if and only if $\alpha = p \beta + q$ is an eigenvalue of $\Mat{A}$.
Let $q = \frac{\operatorname{tr}(\Mat{A})}{3}$ and $p = \left( \trace \left( \frac{(\Mat{A} - q\Mat{I})^2}{6} \right) \right)^{1/2}$, we derive $\det(\beta \Mat{I} - \Mat{B}) = \beta^3 - 3\beta - \det(\Mat{B}) = 0$.
Substitute $\beta = 2 \cos \theta$ and some algebraic simplification using the identity $\cos 3\theta = 4 \cos^3 \theta - 3 \cos \theta$, we can obtain
$\cos 3\theta = \frac{\det(\Mat{B})}{2}$.
Thus, the roots of characteristic equation are given by:
\begin{align*}
\beta = 2 \cos \left( \frac{1}{3} \arccos\left( \frac{\det(\Mat{B})}{2} \right) + \frac{2k\pi}{3} \right), \quad k = 0, 1, 2.    
\end{align*}
When $\Mat{A}$ is real and symmetric, $\det(\Mat{B})$ is also real and no greater than $2$ in absolute value.

\begin{table}[t]
\centering
\resizebox{\linewidth}{!}{
\begin{tabular}{c|cc|cc|ccc|cccc}
\toprule
& \multicolumn{2}{c|}{Tank \& Temple} & \multicolumn{2}{c|}{Deep Blending} & \multicolumn{3}{c|}{mip-NeRF 360 Outdoor} & \multicolumn{4}{c}{mip-NeRF 360 Indoor} \\
 & Train & Truck & Dr. Johnson & Playroom & Bicycle & Garden & Stump & Bonsai & Counter & Kitchen & Room \\
\hline

3DGS & 22.091 & 25.394 & 29.209 & 30.172 & 25.253 & 27.417 & 26.705 & 32.298 & 29.006 & 31.628 & 31.540 \\
SteepGS & 21.974 & 25.395 & 29.478 & 30.447 & 24.890 & 27.159 & 26.115 & 31.911 & 28.737 & 31.030 & 31.401 \\
\bottomrule
\end{tabular}
}
\caption{Breakdown table for per-scene PNSR of 3DGS and our SteepGS.}
\label{tab:breakdown_psnr}
\end{table}

\begin{table}[t]
\centering
\resizebox{\linewidth}{!}{
\begin{tabular}{c|cc|cc|ccc|cccc}
\toprule
& \multicolumn{2}{c|}{Tank \& Temple} & \multicolumn{2}{c|}{Deep Blending} & \multicolumn{3}{c|}{mip-NeRF 360 Outdoor} & \multicolumn{4}{c}{mip-NeRF 360 Indoor} \\
 & Train & Truck & Dr. Johnson & Playroom & Bicycle & Garden & Stump & Bonsai & Counter & Kitchen & Room \\
\hline

3DGS & 0.813 & 0.882 & 0.901 & 0.907 & 0.766 & 0.867 & 0.908 & 0.773 & 0.942 & 0.928 & 0.919 \\
SteepGS & 0.802 & 0.879 & 0.902 & 0.909 & 0.734 & 0.851 & 0.742 & 0.938 & 0.900 & 0.922 & 0.915 \\
\bottomrule
\end{tabular}
}
\caption{Breakdown table for per-scene SSIM of 3DGS and our SteepGS.}
\label{tab:breakdown_ssim}
\end{table}

\begin{table}[t]
\centering
\resizebox{\linewidth}{!}{
\begin{tabular}{c|cc|cc|ccc|cccc}
\toprule
& \multicolumn{2}{c|}{Tank \& Temple} & \multicolumn{2}{c|}{Deep Blending} & \multicolumn{3}{c|}{mip-NeRF 360 Outdoor} & \multicolumn{4}{c}{mip-NeRF 360 Indoor} \\
 & Train & Truck & Dr. Johnson & Playroom & Bicycle & Garden & Stump & Bonsai & Counter & Kitchen & Room \\
\hline

3DGS & 0.207 & 0.147 & 0.244 & 0.244 & 0.210 & 0.107 & 0.215 & 0.203 & 0.200 & 0.126 & 0.219 \\
SteepGS & 0.230 & 0.160 & 0.251 & 0.250 & 0.268 & 0.142 & 0.271 & 0.211 & 0.217 & 0.137 & 0.233\\
\bottomrule
\end{tabular}
}
\caption{Breakdown table for per-scene LPIPS of 3DGS and our SteepGS.}
\label{tab:breakdown_lpips}
\end{table}

\begin{table}[t]
\centering
\resizebox{\linewidth}{!}{
\begin{tabular}{c|cc|cc|ccc|cccc}
\toprule
& \multicolumn{2}{c|}{Tank \& Temple} & \multicolumn{2}{c|}{Deep Blending} & \multicolumn{3}{c|}{mip-NeRF 360 Outdoor} & \multicolumn{4}{c}{mip-NeRF 360 Indoor} \\
 & Train & Truck & Dr. Johnson & Playroom & Bicycle & Garden & Stump & Bonsai & Counter & Kitchen & Room \\
\hline

3DGS & 1088197 & 2572172 & 3316036 & 2320830 & 6074705 & 5845401 & 4863462 & 1260017 & 1195896 & 1807771 & 1550152 \\
SteepGS & 530476 & 1387065 & 1485567 & 1107604 & 2900640 & 2195185 & 3175021 & 746163 & 591508 & 922717 & 710212\\
\bottomrule
\end{tabular}
}
\caption{Breakdown table for the number of points densified by 3DGS and our SteepGS.}
\label{tab:breakdown_pts}
\end{table}

\begin{figure}[t]
  \begin{minipage}[t]{0.45\linewidth}
      \centering
      \includegraphics[width=\textwidth]{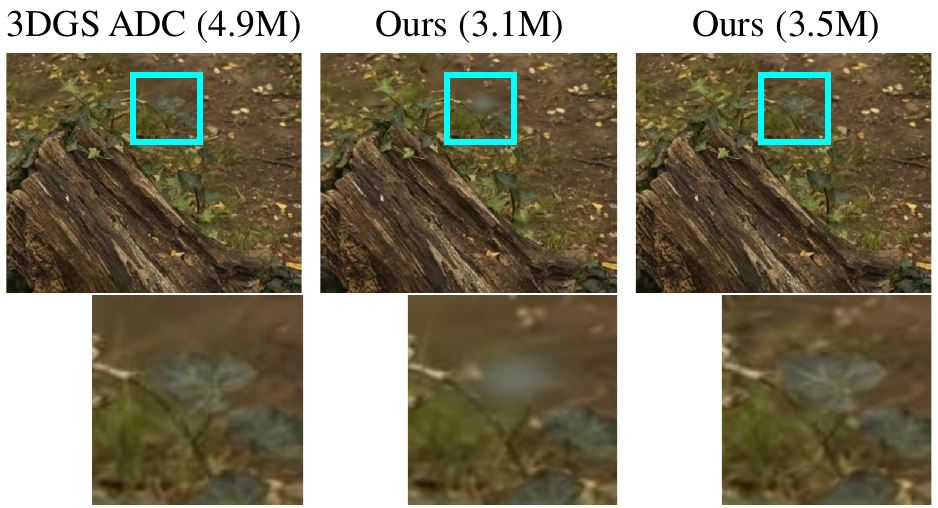}
      \caption{Improved visual quality of our method after more steps of Gaussian splitting.}
      \label{fig:perf-improve}
  \end{minipage}
  \hfill
  \begin{minipage}{0.45\linewidth}
    \centering
    \vspace{-0.52in}
    \resizebox{\linewidth}{!}{
    \begin{tabular}{lccc}
        \toprule
        & Bicycle & Garden & Stump \\ \hline
        3DGS & 25.25 & 27.42 & 26.70 \\
        Ours & 24.89 & 27.16 & 26.11 \\
        Ours (more steps) & 25.23 & 27.38 & 26.65 \\
        \bottomrule
    \end{tabular}}
    \vspace{0.2in}
    \captionsetup{type=table}
    \caption{Improved performance of our method evaluated in PSNR after more steps of Gaussian splitting.}
    \label{tab:perf-improve}
  \end{minipage}
\end{figure}

\section{More Experiment Results}

\paragraph{Metrics Breakdown.} Tables \ref{tab:breakdown_psnr}, \ref{tab:breakdown_ssim}, \ref{tab:breakdown_lpips} and \ref{tab:breakdown_pts} provide breakdown numerical evaluations of PSNR, SSIM, LPIPS, and the number of points for both our method and the original adaptive density control. The results demonstrate that our method achieves performance comparable to the original densification across all scenes. Notably, in the \texttt{Playroom} and \texttt{Dr.~Johnson} scenes, our method outperforms the original adaptive density control while utilizing only half the number of points.

\paragraph{More Visualizations.} Fig. \ref{fig:vis_densify_more} visualizes the points selected for densification in four scenes. It can be observed that our method selects fewer points by concentrating on regions with blurry under-reconstructed areas. In contrast, the original adaptive density control performs more densifications on high-frequency details, which is less likely to effectively enhance rendering quality. These findings validate that our method conserves computational resources by directing densification toward areas that result in the steepest descent in rendering loss.

\paragraph{More Metrics and Compared Methods.} In addition to the compared methods in the main text, we test two more baselines: Compact-3DGS \cite{lee2024compact} and LP-3DGS \cite{zhanglp}. We also include elapsed time on GPU for training, mean and peak GPU memory usage for training, and rendering FPS\footnote{We observed that measuring FPS can be inconsistent, and the values reported in the table should be considered as a reference.} as additional metrics. Table \ref{tab:quan_res_more_metrics} presents the comparison results evaluated on MipNeRF360, Temple\&Tanks, and Deep Blending datasets. Although Compact-3DGS and LP-3DGS yield fewer points in the final results, our method achieves better metrics in PSNR and significantly reduces training time on GPUs. Moreover, our method consistently decreases GPU memory usage and improves rendering FPS compared to the original 3DGS ADC, performing on par with the two newly compared methods.

\paragraph{Improved Performance.} Readers might feel curious if our method could achieve even more closer performance to that of the original 3DGS ADC. In our main experiments, to ensure fair comparisons, we reuse the hyper-parameters of ADC. However, we found that extending the densification iterations to 25K and the total training steps to 40K on some MipNeRF360 scenes allows our method to achieve better performance and further mitigates the blurriness observed in the rendered images. As a reference, Table \ref{tab:perf-improve} demonstrates performance improvements with more densification iterations. Figure \ref{fig:perf-improve} shows reduced blurriness in the \texttt{stump} scene.

\begin{figure*}[t]
  \centering
  \includegraphics[width=\textwidth]{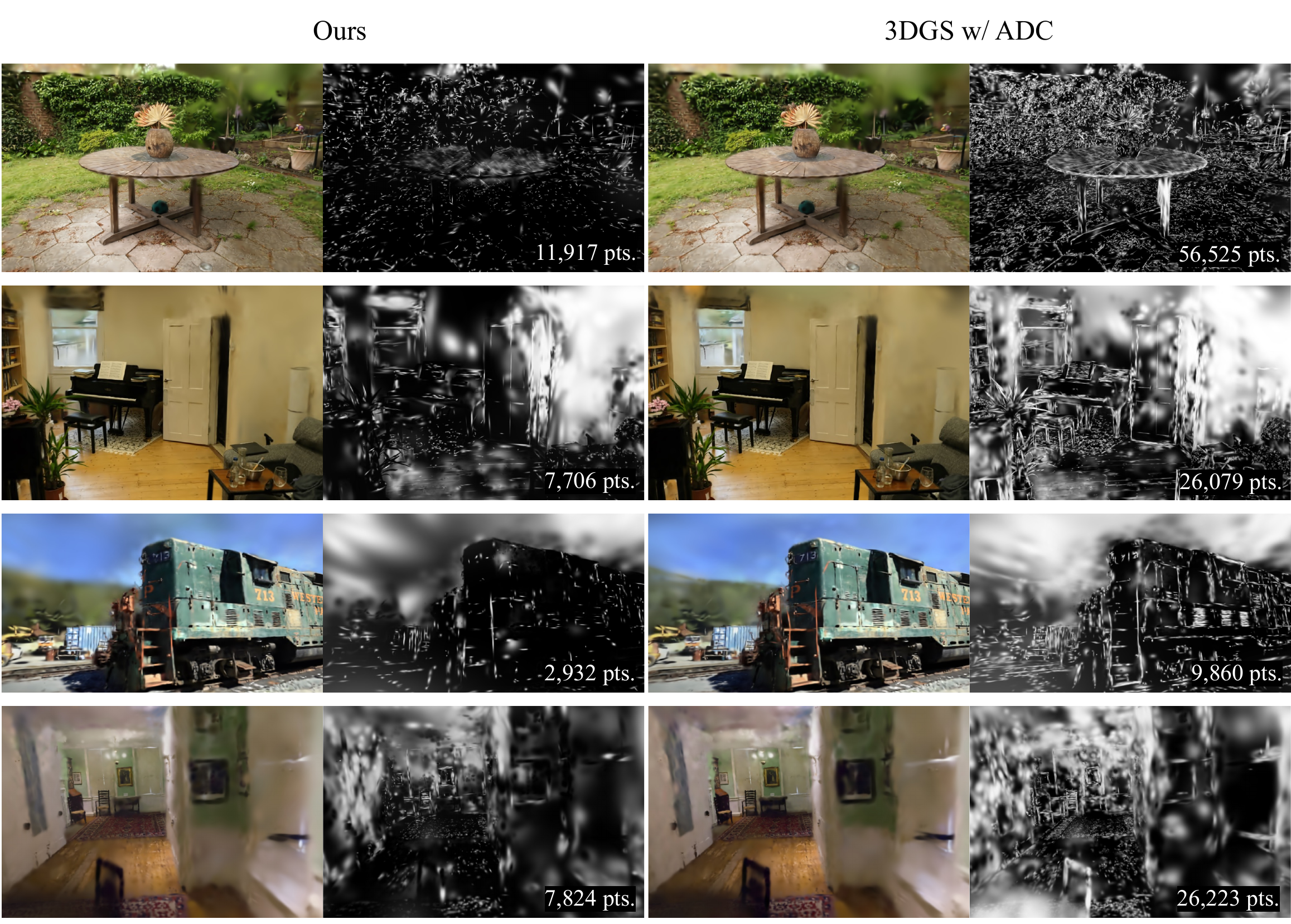}
  \caption{More visualizations of splitting points. We compare the number of points split by our proposed method and the original ADC.}
  \label{fig:vis_densify_more}
\end{figure*}

\begin{table}[t!]
    \centering
    \renewcommand{\arraystretch}{1.5}
    \resizebox{\linewidth}{!}{
    \begin{tabular}{lcccccccc}
        \toprule
        \multicolumn{9}{c}{MipNeRF360} \\
        & \# Points $\downarrow$ & PSNR $\uparrow$ & SSIM $\uparrow$ & LPIPS $\downarrow$ & GPU elapse $\downarrow$ & mean GPU mem. $\downarrow$ & peak GPU mem. $\downarrow$ & FPS $\uparrow$ \\ \hline
        3DGS & 3.339 M & 29.037 & 0.872 & 0.183 & 1550.925 s & 10.262 GB & 12.110 GB & 179 \\ \hline
        LP-3DGS & 1.303 M & 28.640 & 0.865 & 0.198 & 1177.648 s & 10.027 GB & 12.458 GB & 350 \\
        Compact-3DGS & 1.310 M & 28.504 & 0.856 &	0.208 & 4063.203 s & 7.274 GB & 9.044 GB & 98 \\ \hline
        SteepGS (Ours) & 1.606 M & 28.734 & 0.857 & 0.211 & 1051.276 s & 7.597 GB & 8.957 GB & 252 \\ \hline
    \end{tabular}}

    \resizebox{\linewidth}{!}{
    \begin{tabular}{lcccccccc}
        \hline
        \multicolumn{9}{c}{Tank \& Temple} \\
        & \# Points $\downarrow$ & PSNR $\uparrow$ & SSIM $\uparrow$ & LPIPS $\downarrow$ & GPU elapse $\downarrow$ & mean GPU mem. $\downarrow$ & peak GPU mem. $\downarrow$ & FPS $\uparrow$ \\ \hline
        3DGS & 1.830 M & 23.743 & 0.848 & 0.177 & 803.542 s & 5.193 GB & 6.241 GB & 248 \\
        \hline
        LP-3DGS &  0.671 M & 23.424 & 0.839 & 0.197 & 1021.806 s & 5.045 GB & 6.489 GB & 150 \\
        Compact-3DGS & 0.836 M & 23.319 & 0.835 & 0.200 & 1255.748 s & 3.802 GB & 4.774 GB & 357 \\ \hline
        SteepGS (Ours) & 0.958 M & 23.684 & 0.840 & 0.194 & 539.048 s & 4.701 GB & 5.607 GB & 343 \\ \hline
    \end{tabular}}

    \resizebox{\linewidth}{!}{
    \begin{tabular}{lcccccccc}
        \hline
        \multicolumn{9}{c}{Deep Blending} \\
        & \# Points $\downarrow$ & PSNR $\uparrow$ & SSIM $\uparrow$ & LPIPS $\downarrow$ & GPU elapse $\downarrow$ & mean GPU mem. $\downarrow$ & peak GPU mem. $\downarrow$ & FPS $\uparrow$ \\ \hline
        3DGS & 2.818 M & 29.690 & 0.904 & 0.244 & 1429.878 s & 8.668 GB & 10.218 GB & 187 \\ \hline
        LP-3DGS & 0.861 M & 29.764 & 0.906 & 0.249 & 1697.793 s & 8.354 GB & 10.115 GB & 134 \\
        Compact-3DGS & 1.054 M & 29.896 & 0.905 & 0.255 & 1861.897 s & 6.332 GB & 8.026 GB & 312 \\ \hline
        SteepGS (Ours) & 1.296 M & 29.963 & 0.905 & 0.250 & 956.536 s & 5.928 GB & 9.506 GB & 280 \\ \bottomrule
    \end{tabular}}

    \caption{Comparison with LP-3DGS \cite{zhanglp} and Compact-3DGS \cite{lee2024compact} baselines on MipNeRF360, Tank \& Temple, and Deep Blending datasets. Additional metrics: GPU elapsed time for training, mean \& peak GPU memory usage, and FPS are included.}
    \label{tab:quan_res_more_metrics}
\end{table}

\section{Theory}
\label{sec:proofs}

\subsection{Notations and Setup}
To begin with, we re-introduce our notations and the problem setup more rigorously.
We abstract each Gaussian as a function $\sigma_{\Pi}(\Mat{x}; \Mat{\theta}^{(i)}): \Theta 
\times \Set{X} \rightarrow \Set{O}$ where $\Mat{\theta}^{(i)} \in \Theta$ are parameters encapsulating mean, covariance, density, SH coefficients, $(\Pi, \Mat{x}) \in \Set{X}$ denote the camera transformations and the 2D-pixel coordinates respectively, and output includes density and RGB color in space $\Set{O}$.
Further on, we assign the input space a probability measure $\Set{D}(\Set{X})$.
We combine $\alpha$-blending and the photometric loss as a single function $\ell(\cdot): \mathbb{P}(\Set{O}) \mapsto \real$, where $\mathbb{P}(\Set{O})$ denotes the entire output space, i.e., all multisets whose elements are in the output space $\Set{O}$.
Suppose the scene has $n$ Gaussians, then we denote the all parameters as $\Mat{\theta} = \{ \Mat{\theta}^{(i)} \}_{i=1}^{n}$ for shorthand and the total loss function can be expressed as:
\begin{align}
\cL(\Mat{\theta}) = \mean_{\Pi, \Mat{x} \sim \Set{D}(\Set{X})} [\ell(\sigma_{\Pi}(\Mat{x}; \Mat{\theta}^{(1)}), \cdots, \sigma_{\Pi}(\Mat{x}; \Mat{\theta}^{(n)}))].
\end{align}
Now our goal is to split each Gaussian into $m_i$ off-springs.
We denote the parameters of the $i$-th Gaussian's off-springs as $\Mat{\vartheta}^{(i)} = \{\Mat{\vartheta}^{(i)}_j\}_{j=1}^{m_i}$, where $\Mat{\vartheta}^{(i)}_j$ is the $j$-th off-spring of the $i$-th Gaussian and assign it a group of reweighting coefficients $\Mat{w}^{(i)} = \{ w^{(i)}_j \}_{j=1}^{m_i}$ to over-parameterize the original Gaussian as: $\sum_{j=1}^{m_i} w^{(i)}_j \sigma_{\Pi}(\Mat{x}; \Mat{\vartheta}^{(i)}_j)$ such that $\sum_{j=1}^{m_i} w^{(i)}_j = 1$ for every $i \in [n]$
We collect parameters of all the new Gaussians as $\Mat{\vartheta} = \{\Mat{\vartheta}^{(i)}\}_{i=1}^{n}$, and reweighting coefficients as $\Mat{w} = \{ \Mat{w}^{(i)} \}_{i=1}^{n}$, for shorthand.
With newly added Gaussians, the augmented loss function becomes:
\begin{align}
\cL(\Mat{\vartheta}, \Mat{w}) = \mean_{\Pi, \Mat{x} \sim \Set{D}(\Set{X})} \left[\ell\left(\sum_{j=1}^{m_1} w^{(1)}_j \sigma_{\Pi}(\Mat{x}; \Mat{\vartheta}^{(1)}_j), \cdots, \sum_{j=1}^{m_n} w^{(n)}_j \sigma_{\Pi}(\Mat{x}; \Mat{\vartheta}^{(n)}_j) \right) \right].
\end{align}

\subsection{Main Results}

\begin{proof}[Proof of Theorem \ref{thm:sec_ord_approx}]

We define $\Mat{\mu}^{(i)}$ as the average displacement on $\Mat{\theta}^{(i)}$: $\Mat{\mu}^{(i)} = (\sum_{j=1}^{m_i} w^{(i)}_j \Mat{\vartheta}^{(i)}_j - \Mat{\theta}^{(i)}) / \epsilon$ and $\Mat{\delta}^{(i)}_j = (\Mat{\vartheta}^{(i)}_j - \Mat{\theta}^{(i)})/\epsilon - \Mat{\mu}^{(i)}$ as offset additional to $\Mat{\mu}^{(i)}$ for the $j$-th off-spring.
It is obvious that:
\begin{align}
\sum_{j=1}^{m_i} w^{(i)}_j \Mat{\delta}^{(i)}_j &= \sum_{j=1}^{m_i} w^{(i)}_j \left(\frac{\Mat{\vartheta}^{(i)}_j - \Mat{\theta}^{(i)}}{\epsilon} - \Mat{\mu}^{(i)}\right) = \frac{1}{\epsilon} \sum_{j=1}^{m_i} w^{(i)}_j \Mat{\vartheta}^{(i)}_j - \frac{1}{\epsilon} \sum_{j=1}^{m_i} w^{(i)}_j \Mat{\theta}^{(i)} - \sum_{j=1}^{m_i} w^{(i)}_j \Mat{\mu}^{(i)} \nonumber \\
\label{eqn:delta_sum_zero} &= \frac{1}{\epsilon} \sum_{j=1}^{m_i} w^{(i)}_j \Mat{\vartheta}^{(i)}_j - \frac{1}{\epsilon} \Mat{\theta}^{(i)} - \Mat{\mu}^{(i)}
= \Mat{0}.
\end{align}
In addition, we let $\Mat{\Delta}^{(i)}_{j} = \Mat{\mu}^{(i)} + \Mat{\delta}^{(i)}_j$, and
$\Mat{\vartheta}^{(i)}_j$ can be written as: $\Mat{\vartheta}^{(i)}_j = \Mat{\theta}^{(i)} + \epsilon \Mat{\Delta}^{(i)}_j = \Mat{\theta}^{(i)} + \epsilon (\Mat{\mu}^{(i)} + \Mat{\delta}^{(i)}_j)$.
We define an auxiliary function: $\cL(\Mat{\theta}^{(\setminus i)}, \Mat{\vartheta}^{(i)}, \Mat{w}^{(i)})$ as:
\begin{align}
\cL(\Mat{\theta}^{(\setminus i)}, \Mat{\vartheta}^{(i)}, \Mat{w}^{(i)}) = \mean_{\Pi, \Mat{x} \sim \Set{D}(\Set{X})} \left[\ell\left( \sigma_{\Pi}(\Mat{x}; \Mat{\theta}^{(1)}), \cdots, \sum_{j=1}^{m_i} w^{(i)}_j \sigma_{\Pi}(\Mat{x}; \Mat{\vartheta}^{(i)}_j), \cdots, \sigma_{\Pi}(\Mat{x}; \Mat{\theta}^{(n)}) \right) \right],
\end{align}
which only splits the $i$-th Gaussian $\Mat{\theta}^{(i)}$ as $\Mat{\vartheta}^{(i)}$.
By Lemma \ref{lem:split_one_vs_split_all}, we have that:
\begin{align} \label{eqn:split_one_vs_split_all}
\left(\cL(\Mat{\vartheta}, \Mat{w}) - \cL(\Mat{\theta})\right) =
\sum_{i=1}^{n} \left(\cL(\Mat{\theta}^{(\setminus i)}, \Mat{\vartheta}^{(i)}, \Mat{w}^{(i)}) - \cL(\Mat{\theta}) \right) + \frac{\epsilon^2}{2} \sum_{\substack{i, i' \in [n] \\ i \ne i'}} \Mat{\mu}^{(i)\top} \partial^2_{\Mat{\theta}^{(i)} \Mat{\theta}^{(i')}} \cL(\Mat{\theta}) \Mat{\mu}^{(i')} + \bigO(\epsilon^3).
\end{align}
By Lemma \ref{lem:loss_decomp_split_one}, we have:
\begin{align} \label{eqn:loss_decomp_split_one}
\cL(\Mat{\theta}^{(\setminus i)}, \Mat{\vartheta}^{(i)}, \Mat{w}^{(i)}) - \cL(\Mat{\theta}) &= \epsilon \partial_{\Mat{\theta}^{(i)}} \cL(\Mat{\theta})^{\top} \Mat{\mu}^{(i)} + \frac{\epsilon^2}{2} \Mat{\mu}^{(i)\top} 
\partial^2_{\Mat{\theta}^{(i)} \Mat{\theta}^{(i)}} \cL(\Mat{\theta}) \Mat{\mu}^{(i)} \\
& \quad + \frac{\epsilon^2}{2} \sum_{j=1}^{m_i} w^{(i)}_{j} \Mat{\delta}^{(i)\top}_j \Mat{S}^{(i)}(\Mat{\theta}) \Mat{\delta}^{(i)}_j + \bigO(\epsilon^3).
\end{align}
Let $\Mat{\mu} = \begin{bmatrix} \Mat{\mu}^{(1)} & \cdots & \Mat{\mu}^{(n)} \end{bmatrix}$ concatenate the average displacement on all Gaussians. Combining Eq. \ref{eqn:split_one_vs_split_all} and Eq. \ref{eqn:loss_decomp_split_one}, we can conclude:
\begin{align*}
\left(\cL(\Mat{\vartheta}, \Mat{w}) - \cL(\Mat{\theta})\right) &= \sum_{i=1}^{n} \left[ \epsilon \partial_{\Mat{\theta}^{(i)}} \cL(\Mat{\theta})^{\top} \Mat{\mu}^{(i)} + \frac{\epsilon^2}{2} \Mat{\mu}^{(i)\top} 
\partial^2_{\Mat{\theta}^{(i)} \Mat{\theta}^{(i)}} \cL(\Mat{\theta}) \Mat{\mu}^{(i)} + \frac{\epsilon^2}{2} \sum_{j=1}^{m_i} w^{(i)}_{j} \Mat{\delta}^{(i)\top}_j \Mat{S}^{(i)}(\Mat{\theta}) \Mat{\delta}^{(i)}_j \right] \\
& \quad + \frac{\epsilon^2}{2} \sum_{\substack{i, i' \in [n] \\ i \ne i'}} \Mat{\mu}^{(i)\top} \partial^2_{\Mat{\theta}^{(i)} \Mat{\theta}^{(i')}} \cL(\Mat{\theta}) \Mat{\mu}^{(i')} + \bigO(\epsilon^3) \\
&= \epsilon \sum_{i=1}^{n}  \partial_{\Mat{\theta}^{(i)}} \cL(\Mat{\theta})^{\top} \Mat{\mu}^{(i)} + \frac{\epsilon^2}{2} \sum_{\substack{i, i' \in [n]}} \Mat{\mu}^{(i)\top} \partial^2_{\Mat{\theta}^{(i)} \Mat{\theta}^{(i')}} \cL(\Mat{\theta}) \Mat{\mu}^{(i')} \\
& \quad + \frac{\epsilon^2}{2} \sum_{j=1}^{m_i} w^{(i)}_{j} \Mat{\delta}^{(i)\top}_j \Mat{S}^{(i)}(\Mat{\theta}) \Mat{\delta}^{(i)}_j + \bigO(\epsilon^3) \\
&= \epsilon \nabla_{\Mat{\theta}} \cL(\Mat{\theta})^{\top} \Mat{\mu} + \frac{\epsilon^2}{2} \Mat{\mu}^{\top} \nabla^2 \cL(\Mat{\theta}) \Mat{\mu} + \frac{\epsilon^2}{2} \sum_{j=1}^{m_i} w^{(i)}_{j} \Mat{\delta}^{(i)\top}_j \Mat{S}^{(i)}(\Mat{\theta}) \Mat{\delta}^{(i)}_j + \bigO(\epsilon^3),
\end{align*}
as desired.
\end{proof}

\begin{proof}[Proof of Theorem \ref{thm:optim_split}]
By standard variational characterization, we have the following lower bound:
\begin{align*}
\Delta^{(i)}(\Mat{\delta}^{(i)}, \Mat{w}^{(i)}; \Mat{\theta}) := \frac{\epsilon^2}{2} \sum_{j=1}^{m_i} w^{(i)}_j \Mat{\delta}^{(i)\top}_j \Mat{S}^{(i)}(\Mat{\theta}) \Mat{\delta}^{(i)}_j \ge \frac{\epsilon^2}{2} \sum_{j=1}^{m_i} w^{(i)}_j  \lambda_{min}(\Mat{S}^{(i)}(\Mat{\theta})) = \frac{\epsilon^2}{2} \lambda_{min}(\Mat{S}^{(i)}(\Mat{\theta})),
\end{align*}
subject to $\lVert \Mat{\delta}^{(i)}_j \rVert \le 1$.
The equality holds only if $\Mat{\delta}^{(i)}_j$ equals to the smallest eigenvector of $\Mat{S}^{(i)}(\Mat{\theta})$.

Hence, there is no decrease on the loss if $\lambda_{min}(\Mat{S}^{(i)}(\Mat{\theta})) \ge 0$.
Otherwise, we can simply choose $m_i = 2$, $w^{(i)}_1 = w^{(i)}_2 = 1/2$, $\Mat{\delta}^{(i)}_1 = \Mat{v}_{min}(\Mat{S}^{(i)}(\Mat{\theta}))$, and $\Mat{\delta}^{(i)}_2 = -\Mat{v}_{min}(\Mat{S}^{(i)}(\Mat{\theta}))$ to achieve this lower bound.
\end{proof}

\subsection{Auxiliary Results}

\begin{lemma} \label{lem:derivatives_orig}
The following equalities hold for $\cL(\Mat{\theta})$ for every $i \in [n]$
\begin{align*}
\partial_{\Mat{\theta}^{(i)}} \cL(\Mat{\theta}) &= \mean_{\Pi, \Mat{x} \sim \Set{D}(\Set{X})} \left[\partial_{\sigma^{(i)}} \ell\left(\sigma_{\Pi}(\Mat{x}; \Mat{\theta}^{(1)}), \cdots, \sigma_{\Pi}(\Mat{x}; \Mat{\theta}^{(n)}) \right) \nabla \sigma_{\Pi}(\Mat{x}; \Mat{\theta}^{(i)}) \right], \\
\partial^2_{\Mat{\theta}^{(i)}\Mat{\theta}^{(i)}} \cL(\Mat{\theta}) &= \Mat{T}^{(i)}(\Mat{\theta}) + \Mat{S}^{(i)}(\Mat{\theta}), \\
\partial^2_{\Mat{\theta}^{(i)}\Mat{\theta}^{(i')}} \cL(\Mat{\theta}) &= \mean_{\Pi, \Mat{x} \sim \Set{D}(\Set{X})} \left[\partial^2_{\sigma^{(i)}\sigma^{(i')}} \ell\left(\sigma_{\Pi}(\Mat{x}; \Mat{\theta}^{(1)}), \cdots, \sigma_{\Pi}(\Mat{x}; \Mat{\theta}^{(n)}) \right) \nabla \sigma_{\Pi}(\Mat{x}; \Mat{\theta}^{(i)}) \nabla \sigma_{\Pi}(\Mat{x}; \Mat{\theta}^{(i')})^{\top} \right], \forall i' \in [n], i' \ne i,
\end{align*}
where $\Mat{T}^{(i)}(\Mat{\theta}) = \mean_{\Pi, \Mat{x} \sim \Set{D}(\Set{X})} \left[\partial^2_{\sigma^{(i)}\sigma^{(i)}} \ell\left(\sigma_{\Pi}(\Mat{x}; \Mat{\theta}^{(1)}), \cdots, \sigma_{\Pi}(\Mat{x}; \Mat{\theta}^{(n)}) \right) \nabla \sigma_{\Pi}(\Mat{x}; \Mat{\theta}^{(i)}) \nabla \sigma_{\Pi}(\Mat{x}; \Mat{\theta}^{(i)})^{\top} \right]$.
\end{lemma}
\begin{proof}
The gradient of $\cL(\Mat{\theta})$ is proved via simple chain rule.
And then
\begin{align*}
&\partial^2_{\Mat{\theta}^{(i)}\Mat{\theta}^{(i')}} \cL(\Mat{\theta}) = \partial_{\Mat{\theta}^{(i')}} \left[\partial_{\Mat{\theta}^{(i)}} \cL(\Mat{\theta})\right] \\
&= \mean_{\Pi, \Mat{x} \sim \Set{D}(\Set{X})} \left[\partial^2_{\sigma^{(i)}\sigma^{(i')}} \ell\left(\sigma_{\Pi}(\Mat{x}; \Mat{\theta}^{(1)}), \cdots, \sigma_{\Pi}(\Mat{x}; \Mat{\theta}^{(n)}) \right) \nabla \sigma_{\Pi}(\Mat{x}; \Mat{\theta}^{(i)}) \nabla \sigma_{\Pi}(\Mat{x}; \Mat{\theta}^{(i')})^{\top} \right] \\
&\quad + \mean_{\Pi, \Mat{x} \sim \Set{D}(\Set{X})} \left[\partial_{\sigma^{(i)}} \ell\left(\sigma_{\Pi}(\Mat{x}; \Mat{\theta}^{(1)}), \cdots, \sigma_{\Pi}(\Mat{x}; \Mat{\theta}^{(n)}) \right) \partial_{\Mat{\theta}^{(i')}} \nabla \sigma_{\Pi}(\Mat{x}; \Mat{\theta}^{(i)}) \right].
\end{align*}
When $i = i'$, $\partial_{\Mat{\theta}^{(i')}} \nabla \sigma_{\Pi}(\Mat{x}; \Mat{\theta}^{(i)}) = \nabla^2 \sigma_{\Pi}(\Mat{x}; \Mat{\theta}^{(i)})$, henceforth:
\begin{align*}
\partial^2_{\Mat{\theta}^{(i)}\Mat{\theta}^{(i)}} \cL(\Mat{\theta}) &= \Mat{T}^{(i)}(\Mat{\theta}) + \Mat{S}^{(i)}(\Mat{\theta}).
\end{align*}
Otherwise, $\partial_{\Mat{\theta}^{(i')}} \nabla \sigma_{\Pi}(\Mat{x}; \Mat{\theta}^{(i)}) = \Mat{0}$, and thus:
\begin{align*}
\partial^2_{\Mat{\theta}^{(i)}\Mat{\theta}^{(i')}} \cL(\Mat{\theta}) &= \mean_{\Pi, \Mat{x} \sim \Set{D}(\Set{X})} \left[\partial^2_{\sigma^{(i)}\sigma^{(i')}} \ell\left(\sigma_{\Pi}(\Mat{x}; \Mat{\theta}^{(1)}), \cdots, \sigma_{\Pi}(\Mat{x}; \Mat{\theta}^{(n)}) \right) \nabla \sigma_{\Pi}(\Mat{x}; \Mat{\theta}^{(i)}) \nabla \sigma_{\Pi}(\Mat{x}; \Mat{\theta}^{(i')})^{\top} \right],
\end{align*}
all as desired.
\end{proof}

\begin{lemma} \label{lem:derivatives_split_all}
The following equalities hold for $\cL(\Mat{\vartheta}, \Mat{w})$ at $\epsilon = 0$:
\begin{align}
\left.\partial_{\Mat{\vartheta}^{(i)}_j} \cL(\Mat{\vartheta}, \Mat{w}) \right\vert_{\epsilon=0} &= w^{(i)}_j \partial_{\Mat{\theta}^{(i)}} \cL(\Mat{\theta}), \quad \forall i \in [n], j \in [m_i], \\
\left.\partial^2_{\Mat{\vartheta}^{(i)}_j \Mat{\vartheta}^{(i)}_{j}} \cL(\Mat{\vartheta}, \Mat{w}) \right\vert_{\epsilon=0} &= w^{(i)}_j \Mat{S}^{(i)}(\Mat{\theta}) + w^{(i)2}_j \Mat{T}^{(i)}(\Mat{\theta}), \quad \forall i \in [n], j \in [m_i], \\
\left.\partial^2_{\Mat{\vartheta}^{(i)}_j \Mat{\vartheta}^{(i)}_{j'}} \cL(\Mat{\vartheta}, \Mat{w}) \right\vert_{\epsilon=0} &= w^{(i)}_j  w^{(i)}_{j'} \Mat{T}^{(i)}(\Mat{\theta}), \quad \forall i \in [n], j, j' \in [m_i], j \ne j', \\
\left.\partial^2_{\Mat{\vartheta}^{(i)}_j, \Mat{\vartheta}^{(i')}_{j'}} \cL(\Mat{\vartheta}, \Mat{w}) \right\vert_{\epsilon=0} &= w^{(i)}_j  w^{(i')}_{j'} \partial^2_{\Mat{\theta}^{(i)}\Mat{\theta}^{(i')}} \cL(\Mat{\theta}), \quad \forall i, i' \in [n], i \ne i', j \in [m_i], j' \in [m_{i'}], \label{eqn:derivatives_split_all_cross}
\end{align}
where $\Mat{T}^{(i)}(\Mat{\theta}) = \mean_{\Pi, \Mat{x} \sim \Set{D}(\Set{X})} \left[\partial^2_{\sigma^{(i)}\sigma^{(i)}} \ell\left(\sigma_{\Pi}(\Mat{x}; \Mat{\theta}^{(1)}), \cdots, \sigma_{\Pi}(\Mat{x}; \Mat{\theta}^{(n)}) \right) \nabla \sigma_{\Pi}(\Mat{x}; \Mat{\theta}^{(i)}) \nabla \sigma_{\Pi}(\Mat{x}; \Mat{\theta}^{(i)})^{\top} \right]$ is as defined in Lemma \ref{lem:derivatives_orig}.
\end{lemma}
\begin{proof}
Let $\wt{\sigma}_{\Pi}(\Mat{x}; \Mat{\vartheta}^{(i)}) = \sum_{j=1}^{m_i} w^{(i)}_j \sigma_{\Pi}(\Mat{x}; \Mat{\vartheta}^{(i)}_j)$ and we can express $\cL(\Mat{\vartheta}, \Mat{w})$ as:
\begin{align}
\cL(\Mat{\vartheta}, \Mat{w}) = \mean_{\Pi, \Mat{x} \sim \Set{D}(\Set{X})} \left[\ell\left(\wt{\sigma}_{\Pi}(\Mat{x}; \Mat{\vartheta}^{(1)}), \cdots, \wt{\sigma}_{\Pi}(\Mat{x}; \Mat{\vartheta}^{(n)}) \right) \right].
\end{align}
To take derivatives of $\cL(\Mat{\vartheta}, \Mat{w})$, we leverage the chain rule. For every $i \in [n], j \in [m_i]$:
\begin{align}
\partial_{\Mat{\vartheta}^{(i)}_j} \cL(\Mat{\vartheta}, \Mat{w})
&= \partial_{\Mat{\vartheta}^{(i)}_j} \mean_{\Pi, \Mat{x} \sim \Set{D}(\Set{X})} \left[\ell\left(\wt{\sigma}_{\Pi}(\Mat{x}; \Mat{\vartheta}^{(1)}), \cdots, \wt{\sigma}_{\Pi}(\Mat{x}; \Mat{\vartheta}^{(n)}) \right) \right] \nonumber \\
&= \mean_{\Pi, \Mat{x} \sim \Set{D}(\Set{X})} \left[\partial_{\Mat{\vartheta}^{(i)}_j} \ell\left(\wt{\sigma}_{\Pi}(\Mat{x}; \Mat{\vartheta}^{(1)}), \cdots, \wt{\sigma}_{\Pi}(\Mat{x}; \Mat{\vartheta}^{(n)}) \right) \right] \nonumber \\
&= \mean_{\Pi, \Mat{x} \sim \Set{D}(\Set{X})} \left[\partial_{\sigma^{(i)}} \ell\left(\wt{\sigma}_{\Pi}(\Mat{x}; \Mat{\vartheta}^{(1)}), \cdots, \wt{\sigma}_{\Pi}(\Mat{x}; \Mat{\vartheta}^{(n)}) \right) \partial_{\Mat{\vartheta}^{(i)}_{j}} \wt{\sigma}_{\Pi}(\Mat{x}; \Mat{\vartheta}^{(i)})  \right] \nonumber \\
&= w^{(i)}_j \mean_{\Pi, \Mat{x} \sim \Set{D}(\Set{X})} \left[\partial_{\sigma^{(i)}} \ell\left(\wt{\sigma}_{\Pi}(\Mat{x}; \Mat{\vartheta}^{(1)}), \cdots, \wt{\sigma}_{\Pi}(\Mat{x}; \Mat{\vartheta}^{(n)}) \right) \nabla \sigma_{\Pi}(\Mat{x}; \Mat{\vartheta}^{(i)}_j)  \right]. \label{eqn:d_vartheta_chain_split_all}
\end{align}
Since $\epsilon = 0$, we have $\Mat{\vartheta}^{(i)}_j = \Mat{\theta}^{(i)}$ and $\wt{\sigma}_{\Pi}(\Mat{x}; \Mat{\vartheta}^{(i)}) = \sum_{j=1}^{m_i} w^{(i)}_j \sigma_{\Pi}(\Mat{x}; \Mat{\theta}^{(i)}) = \sigma_{\Pi}(\Mat{x}; \Mat{\theta})$.
Hence, we can further simplify Eq. \ref{eqn:d_vartheta_chain_split_all} as:
\begin{align*}
\left.\partial_{\Mat{\vartheta}^{(i)}_j} \cL(\Mat{\vartheta}, \Mat{w}) \right\vert_{\epsilon=0} &= w^{(i)}_j \left.\mean_{\Pi, \Mat{x} \sim \Set{D}(\Set{X})} \left[\partial_{\sigma^{(i)}} \ell\left(\wt{\sigma}_{\Pi}(\Mat{x}; \Mat{\vartheta}^{(1)}), \cdots, \wt{\sigma}_{\Pi}(\Mat{x}; \Mat{\vartheta}^{(n)}) \right) \nabla \sigma_{\Pi}(\Mat{x}; \Mat{\vartheta}^{(i)}_j)  \right] \right\vert_{\epsilon=0} \\
&= w^{(i)}_j \mean_{\Pi, \Mat{x} \sim \Set{D}(\Set{X})} \left[\partial_{\sigma^{(i)}} \ell\left(\sigma_{\Pi}(\Mat{x}; \Mat{\theta}^{(1)}), \cdots, \sigma_{\Pi}(\Mat{x}; \Mat{\theta}^{(n)}) \right)  \nabla \sigma_{\Pi}(\Mat{x}; \Mat{\theta}^{(i)}) \right] \\
&= w^{(i)}_j \partial_{\Mat{\theta}^{(i)}} \cL(\Mat{\theta}),
\end{align*}
where the last step is due to Lemma \ref{lem:derivatives_orig}.

Next we derive second-order derivatives.
Taking derivatives of Eq. \ref{eqn:d_vartheta_chain_split_all} in terms of $\Mat{\vartheta}^{(i')}_{j'}$ for some $i' \in [n], j' \in [m_{i'}]$, and by chain rule:
\begin{align*}
&\partial^2_{\Mat{\vartheta}^{(i)}_j\Mat{\vartheta}^{(i')}_{j'}} \cL(\Mat{\vartheta}, \Mat{w}) = \partial_{\Mat{\vartheta}^{(i')}_{j'}}  w^{(i)}_j \mean_{\Pi, \Mat{x} \sim \Set{D}(\Set{X})} \left[\partial_{\sigma^{(i)}} \ell\left(\wt{\sigma}_{\Pi}(\Mat{x}; \Mat{\vartheta}^{(1)}), \cdots, \wt{\sigma}_{\Pi}(\Mat{x}; \Mat{\vartheta}^{(n)}) \right) \nabla \sigma_{\Pi}(\Mat{x}; \Mat{\vartheta}^{(i)}_j)  \right] \\
&= w^{(i)}_j \mean_{\Pi, \Mat{x} \sim \Set{D}(\Set{X})} \left[\partial^2_{\sigma^{(i)}\sigma^{(i')}} \ell\left(\wt{\sigma}_{\Pi}(\Mat{x}; \Mat{\vartheta}^{(1)}), \cdots, \wt{\sigma}_{\Pi}(\Mat{x}; \Mat{\vartheta}^{(n)}) \right) \nabla \sigma_{\Pi}(\Mat{x}; \Mat{\vartheta}^{(i)}_j) \partial_{\Mat{\vartheta}^{(i')}_{j'}} \wt{\sigma}_{\Pi}(\Mat{x}; \Mat{\vartheta}^{(i')})^{\top} \right] \\
&\quad + w^{(i)}_j \mean_{\Pi, \Mat{x} \sim \Set{D}(\Set{X})} \left[\partial_{\sigma^{(i)}} \ell\left(\wt{\sigma}_{\Pi}(\Mat{x}; \Mat{\vartheta}^{(1)}), \cdots, \wt{\sigma}_{\Pi}(\Mat{x}; \Mat{\vartheta}^{(n)}) \right) \partial_{\Mat{\vartheta}^{(i')}_{j'}} \nabla \sigma_{\Pi}(\Mat{x}; \Mat{\vartheta}^{(i)}_j) \right] \\
&= w^{(i)}_j w^{(i')}_{j'} \mean_{\Pi, \Mat{x} \sim \Set{D}(\Set{X})} \left[\partial^2_{\sigma^{(i)}\sigma^{(i')}} \ell\left(\wt{\sigma}_{\Pi}(\Mat{x}; \Mat{\vartheta}^{(1)}), \cdots, \wt{\sigma}_{\Pi}(\Mat{x}; \Mat{\vartheta}^{(n)}) \right) \nabla \sigma_{\Pi}(\Mat{x}; \Mat{\vartheta}^{(i)}_j) \nabla \sigma_{\Pi}(\Mat{x}; \Mat{\vartheta}^{(i')}_{j'})^{\top} \right] \\
&\quad + w^{(i)}_j \mean_{\Pi, \Mat{x} \sim \Set{D}(\Set{X})} \left[\partial_{\sigma^{(i)}} \ell\left(\wt{\sigma}_{\Pi}(\Mat{x}; \Mat{\vartheta}^{(1)}), \cdots, \wt{\sigma}_{\Pi}(\Mat{x}; \Mat{\vartheta}^{(n)}) \right) \partial_{\Mat{\vartheta}^{(i')}_{j'}} \nabla \sigma_{\Pi}(\Mat{x}; \Mat{\vartheta}^{(i)}_j) \right].
\end{align*}
Now we discuss three scenarios:
\begin{enumerate}
\item When $i = i'$ and $j = j'$, $\partial_{\Mat{\vartheta}^{(i')}_{j'}} \nabla \sigma_{\Pi}(\Mat{x}; \Mat{\vartheta}^{(i)}_j) = \nabla^2 \sigma_{\Pi}(\Mat{x}; \Mat{\vartheta}^{(i)}_j)$, and then
\begin{align}
\partial^2_{\Mat{\vartheta}^{(i)}_j\Mat{\vartheta}^{(i)}_{j}} \cL(\Mat{\vartheta}, \Mat{w}) &= w^{(i)2}_j \mean_{\Pi, \Mat{x} \sim \Set{D}(\Set{X})} \left[\partial^2_{\sigma^{(i)}\sigma^{(i)}} \ell\left(\wt{\sigma}_{\Pi}(\Mat{x}; \Mat{\vartheta}^{(1)}), \cdots, \wt{\sigma}_{\Pi}(\Mat{x}; \Mat{\vartheta}^{(n)}) \right) \nabla \sigma_{\Pi}(\Mat{x}; \Mat{\vartheta}^{(i)}_j) \nabla \sigma_{\Pi}(\Mat{x}; \Mat{\vartheta}^{(i)}_j)^{\top} \right] \nonumber \\
&\quad + w^{(i)}_j \mean_{\Pi, \Mat{x} \sim \Set{D}(\Set{X})} \left[\partial_{\sigma^{(i)}} \ell\left(\wt{\sigma}_{\Pi}(\Mat{x}; \Mat{\vartheta}^{(1)}), \cdots, \wt{\sigma}_{\Pi}(\Mat{x}; \Mat{\vartheta}^{(n)}) \right) \nabla^2 \sigma_{\Pi}(\Mat{x}; \Mat{\vartheta}^{(i)}_j) \right] \label{eqn:d2_vartheta_chain_split_all_same_ij}
\end{align}

\item When $i = i'$ and $j \ne j'$, $\partial_{\Mat{\vartheta}^{(i')}_{j'}} \nabla \sigma_{\Pi}(\Mat{x}; \Mat{\vartheta}^{(i)}_j) = \Mat{0}$, and thus
\begin{align}
\partial^2_{\Mat{\vartheta}^{(i)}_j\Mat{\vartheta}^{(i)}_{j'}} \cL(\Mat{\vartheta}, \Mat{w}) &= w^{(i)}_j w^{(i)}_{j'} \mean_{\Pi, \Mat{x} \sim \Set{D}(\Set{X})} \left[\partial^2_{\sigma^{(i)}\sigma^{(i)}} \ell\left(\wt{\sigma}_{\Pi}(\Mat{x}; \Mat{\vartheta}^{(1)}), \cdots, \wt{\sigma}_{\Pi}(\Mat{x}; \Mat{\vartheta}^{(n)}) \right) \nabla \sigma_{\Pi}(\Mat{x}; \Mat{\vartheta}^{(i)}_j) \nabla \sigma_{\Pi}(\Mat{x}; \Mat{\vartheta}^{(i)}_{j'})^{\top} \right] \label{eqn:d2_vartheta_chain_split_all_same_i}
\end{align}

\item When $i \ne i'$, $\partial_{\Mat{\vartheta}^{(i')}_{j'}} \nabla \sigma_{\Pi}(\Mat{x}; \Mat{\vartheta}^{(i)}_j) = \Mat{0}$, and henceforth
\begin{align}
\partial^2_{\Mat{\vartheta}^{(i)}_j\Mat{\vartheta}^{(i')}_{j'}} \cL(\Mat{\vartheta}, \Mat{w}) &= w^{(i)}_j w^{(i')}_{j'} \mean_{\Pi, \Mat{x} \sim \Set{D}(\Set{X})} \left[\partial^2_{\sigma^{(i)}\sigma^{(i')}} \ell\left(\wt{\sigma}_{\Pi}(\Mat{x}; \Mat{\vartheta}^{(1)}), \cdots, \wt{\sigma}_{\Pi}(\Mat{x}; \Mat{\vartheta}^{(n)}) \right) \nabla \sigma_{\Pi}(\Mat{x}; \Mat{\vartheta}^{(i)}_j) \nabla \sigma_{\Pi}(\Mat{x}; \Mat{\vartheta}^{(i')}_{j'})^{\top} \right] \label{eqn:d2_vartheta_chain_split_all_no_same}
\end{align}
\end{enumerate}
Using this fact again: $\Mat{\vartheta}^{(i)}_j = \Mat{\theta}^{(i)}$ and $\wt{\sigma}_{\Pi}(\Mat{x}; \Mat{\vartheta}^{(i)}) = \sum_{j=1}^{m_i} w^{(i)}_j \sigma_{\Pi}(\Mat{x}; \Mat{\theta}^{(i)}) = \sigma_{\Pi}(\Mat{x}; \Mat{\theta}^{(i)})$ when $\epsilon = 0$, Eq. \ref{eqn:d2_vartheta_chain_split_all_same_ij} becomes:
\begin{align*}
\left.\partial^2_{\Mat{\vartheta}^{(i)}_j\Mat{\vartheta}^{(i)}_{j}} \cL(\Mat{\vartheta}, \Mat{w})\right\vert_{\epsilon=0} &= w^{(i)2}_j \mean_{\Pi, \Mat{x} \sim \Set{D}(\Set{X})} \left[\partial^2_{\sigma^{(i)}\sigma^{(i)}} \ell\left(\sigma_{\Pi}(\Mat{x}; \Mat{\theta}^{(1)}), \cdots, \sigma_{\Pi}(\Mat{x}; \Mat{\theta}^{(n)}) \right) \nabla \sigma_{\Pi}(\Mat{x}; \Mat{\theta}^{(i)}) \nabla \sigma_{\Pi}(\Mat{x}; \Mat{\theta}^{(i)})^{\top} \right] \\
&\quad + w^{(i)}_j \mean_{\Pi, \Mat{x} \sim \Set{D}(\Set{X})} \left[\partial_{\sigma^{(i)}} \ell\left(\sigma_{\Pi}(\Mat{x}; \Mat{\theta}^{(1)}), \cdots, \sigma_{\Pi}(\Mat{x}; \Mat{\theta}^{(n)}) \right) \nabla^2 \sigma_{\Pi}(\Mat{x}; \Mat{\theta}^{(i)}) \right] \\
&= w^{(i)2}_j \Mat{T}^{(i)}(\Mat{\theta}) + w^{(i)}_j \Mat{S}^{(i)}(\Mat{\theta}),
\end{align*}
Eq. \ref{eqn:d2_vartheta_chain_split_all_same_i} can be simplified as:
\begin{align*}
\left.\partial^2_{\Mat{\vartheta}^{(i)}_j\Mat{\vartheta}^{(i)}_{j'}} \cL(\Mat{\vartheta}, \Mat{w})\right\vert_{\epsilon=0} &= w^{(i)}_j w^{(i)}_{j'} \mean_{\Pi, \Mat{x} \sim \Set{D}(\Set{X})} \left[\partial^2_{\sigma^{(i)}\sigma^{(i)}} \ell\left(\sigma_{\Pi}(\Mat{x}; \Mat{\theta}^{(1)}), \cdots, \sigma_{\Pi}(\Mat{x}; \Mat{\theta}^{(n)}) \right) \nabla \sigma_{\Pi}(\Mat{x}; \Mat{\theta}^{(i)}) \nabla \sigma_{\Pi}(\Mat{x}; \Mat{\theta}^{(i)})^{\top} \right] \\
&= w^{(i)}_j w^{(i)}_{j'} \Mat{T}^{(i)}(\Mat{\theta}),
\end{align*}
and by Lemma \ref{lem:derivatives_orig}, Eq. \ref{eqn:d2_vartheta_chain_split_all_no_same} turns into:
\begin{align*}
\left.\partial^2_{\Mat{\vartheta}^{(i)}_j\Mat{\vartheta}^{(i')}_{j'}} \cL(\Mat{\vartheta}, \Mat{w})\right\vert_{\epsilon=0} &= w^{(i)}_j w^{(i')}_{j'} \mean_{\Pi, \Mat{x} \sim \Set{D}(\Set{X})} \left[\partial^2_{\sigma^{(i)}\sigma^{(i')}} \ell\left(\sigma_{\Pi}(\Mat{x}; \Mat{\theta}^{(1)}), \cdots, \sigma_{\Pi}(\Mat{x}; \Mat{\theta}^{(n)}) \right) \nabla \sigma_{\Pi}(\Mat{x}; \Mat{\theta}^{(i)}) \nabla \sigma_{\Pi}(\Mat{x}; \Mat{\theta}^{(i')})^{\top} \right] \\
&= w^{(i)}_j w^{(i')}_{j'} \partial^2_{\Mat{\vartheta}^{(i)}_j\Mat{\vartheta}^{(i')}_{j'}} \cL(\Mat{\theta}),
\end{align*}
all as desired.
\end{proof}

\begin{lemma} \label{lem:derivatives_split_one}
The following equalities hold for $\cL(\Mat{\theta}^{(\setminus i)}, \Mat{\vartheta}^{(i)}, \Mat{w}^{(i)})$ at $\epsilon = 0$ for any $i \in [n]$:
\begin{align}
\left.\partial_{\Mat{\vartheta}^{(i)}_j} \cL(\Mat{\theta}^{(\setminus i)}, \Mat{\vartheta}^{(i)}, \Mat{w}^{(i)}) \right\vert_{\epsilon=0} &= w^{(i)}_j \partial_{\Mat{\theta}^{(i)}} \cL(\Mat{\theta}), \quad \forall j \in [m_i], \\
\left.\partial^2_{\Mat{\vartheta}^{(i)}_j \Mat{\vartheta}^{(i)}_{j}} \cL(\Mat{\theta}^{(\setminus i)}, \Mat{\vartheta}^{(i)}, \Mat{w}^{(i)}) \right\vert_{\epsilon=0} &= w^{(i)}_j \Mat{S}^{(i)}(\Mat{\theta}) + w^{(i)2}_j \Mat{T}^{(i)}(\Mat{\theta}), \quad \forall j \in [m_i], \\
\left.\partial^2_{\Mat{\vartheta}^{(i)}_j \Mat{\vartheta}^{(i)}_{j'}} \cL(\Mat{\theta}^{(\setminus i)}, \Mat{\vartheta}^{(i)}, \Mat{w}^{(i)}) \right\vert_{\epsilon=0} &= w^{(i)}_j  w^{(i)}_{j'} \Mat{T}^{(i)}(\Mat{\theta}), \quad \forall j, j' \in [m_i], j \ne j',
\end{align}
where $\Mat{T}^{(i)}(\Mat{\theta}) = \mean_{\Pi, \Mat{x} \sim \Set{D}(\Set{X})} \left[\partial^2_{\sigma^{(i)}\sigma^{(i)}} \ell\left(\sigma_{\Pi}(\Mat{x}; \Mat{\theta}^{(1)}), \cdots, \sigma_{\Pi}(\Mat{x}; \Mat{\theta}^{(n)}) \right) \nabla \sigma_{\Pi}(\Mat{x}; \Mat{\theta}^{(i)}) \nabla \sigma_{\Pi}(\Mat{x}; \Mat{\theta}^{(i)})^{\top} \right]$ is as defined in Lemma \ref{lem:derivatives_orig}.
\end{lemma}
\begin{proof}
The proof is identical to Lemma \ref{lem:derivatives_split_all}.
We outline the details for completeness.
Let $\wt{\sigma}_{\Pi}(\Mat{x}; \Mat{\vartheta}^{(i)}) = \sum_{j=1}^{m_i} w^{(i)}_j \sigma_{\Pi}(\Mat{x}; \Mat{\vartheta}^{(i)}_j)$ and we can express $\cL(\Mat{\theta}^{(\setminus i)}, \Mat{\vartheta}^{(i)}, \Mat{w}^{(i)})$ as:
\begin{align}
\cL(\Mat{\theta}^{(\setminus i)}, \Mat{\vartheta}^{(i)}, \Mat{w}^{(i)}) = \mean_{\Pi, \Mat{x} \sim \Set{D}(\Set{X})} \left[\ell\left(\sigma_{\Pi}(\Mat{x}; \Mat{\theta}^{(1)}), \cdots, \wt{\sigma}_{\Pi}(\Mat{x}; \Mat{\vartheta}^{(i)}), \cdots, \sigma_{\Pi}(\Mat{x}; \Mat{\theta}^{(n)}) \right) \right].
\end{align}
By chain rule, for every $j \in [m_i]$:
\begin{align}
&\partial_{\Mat{\vartheta}^{(i)}_j} \cL(\Mat{\theta}^{(\setminus i)}, \Mat{\vartheta}^{(i)}, \Mat{w}^{(i)})
= \partial_{\Mat{\vartheta}^{(i)}_j} \mean_{\Pi, \Mat{x} \sim \Set{D}(\Set{X})} \left[\ell\left( \sigma_{\Pi}(\Mat{x}; \Mat{\theta}^{(1)}), \cdots, \wt{\sigma}_{\Pi}(\Mat{x}; \Mat{\vartheta}^{(i)}), \cdots, \sigma_{\Pi}(\Mat{x}; \Mat{\theta}^{(n)}) \right) \right] \nonumber \\
&= \mean_{\Pi, \Mat{x} \sim \Set{D}(\Set{X})} \left[\partial_{\sigma^{(i)}} \ell\left(\sigma_{\Pi}(\Mat{x}; \Mat{\theta}^{(1)}), \cdots, \wt{\sigma}_{\Pi}(\Mat{x}; \Mat{\vartheta}^{(i)}), \cdots, \sigma_{\Pi}(\Mat{x}; \Mat{\theta}^{(n)}) \right) \partial_{\Mat{\vartheta}^{(i)}_{j}} \wt{\sigma}_{\Pi}(\Mat{x}; \Mat{\vartheta}^{(i)})  \right] \nonumber \\
&= w^{(i)}_j \mean_{\Pi, \Mat{x} \sim \Set{D}(\Set{X})} \left[\partial_{\sigma^{(i)}} \ell\left(\sigma_{\Pi}(\Mat{x}; \Mat{\theta}^{(1)}), \cdots, \wt{\sigma}_{\Pi}(\Mat{x}; \Mat{\vartheta}^{(i)}), \cdots, \sigma_{\Pi}(\Mat{x}; \Mat{\theta}^{(n)}) \right) \nabla \sigma_{\Pi}(\Mat{x}; \Mat{\vartheta}^{(i)}_j)  \right]. \label{eqn:d_vartheta_chain_split_one}
\end{align}
Using the fact that $\Mat{\vartheta}^{(i)}_j = \Mat{\theta}^{(i)}$ and $\wt{\sigma}_{\Pi}(\Mat{x}; \Mat{\vartheta}^{(i)}) = \sum_{j=1}^{m_i} w^{(i)}_j \sigma_{\Pi}(\Mat{x}; \Mat{\theta}^{(i)}) = \sigma_{\Pi}(\Mat{x}; \Mat{\theta}^{(i)})$ when $\epsilon = 0$, Eq. \ref{eqn:d_vartheta_chain_split_one} can be rewritten as:
\begin{align*}
\left.\partial_{\Mat{\vartheta}^{(i)}_j} \cL(\Mat{\theta}^{(\setminus i)}, \Mat{\vartheta}^{(i)}, \Mat{w}^{(i)}) \right\vert_{\epsilon=0} &= w^{(i)}_j \left.\mean_{\Pi, \Mat{x} \sim \Set{D}(\Set{X})} \left[\partial_{\sigma^{(i)}} \ell\left(\sigma_{\Pi}(\Mat{x}; \Mat{\theta}^{(1)}), \cdots, \wt{\sigma}_{\Pi}(\Mat{x}; \Mat{\vartheta}^{(i)}), \cdots, \sigma_{\Pi}(\Mat{x}; \Mat{\theta}^{(n)}) \right) \nabla \sigma_{\Pi}(\Mat{x}; \Mat{\vartheta}^{(i)}_j)  \right] \right\vert_{\epsilon=0} \\
&= w^{(i)}_j \mean_{\Pi, \Mat{x} \sim \Set{D}(\Set{X})} \left[\partial_{\sigma^{(i)}} \ell\left(\sigma_{\Pi}(\Mat{x}; \Mat{\theta}^{(1)}), \cdots, \sigma_{\Pi}(\Mat{x}; \Mat{\theta}^{(i)}), \cdots, \sigma_{\Pi}(\Mat{x}; \Mat{\theta}^{(n)}) \right)  \nabla \sigma_{\Pi}(\Mat{x}; \Mat{\theta}^{(i)}) \right] \\
&= w^{(i)}_j \partial_{\Mat{\theta}^{(i)}} \cL(\Mat{\theta}),
\end{align*}
where the last step is due to Lemma \ref{lem:derivatives_orig}.

Next we derive second-order derivatives.
Taking derivatives of Eq. \ref{eqn:d_vartheta_chain_split_one} in terms of $\Mat{\vartheta}^{(i)}_{j'}$ for some $j' \in [m_i]$, and by chain rule:
\begin{align*}
&\partial^2_{\Mat{\vartheta}^{(i)}_j\Mat{\vartheta}^{(i)}_{j'}} \cL(\Mat{\theta}^{(\setminus i)}, \Mat{\vartheta}^{(i)}, \Mat{w}^{(i)}) = \partial_{\Mat{\vartheta}^{(i)}_{j'}}  w^{(i)}_j \mean_{\Pi, \Mat{x} \sim \Set{D}(\Set{X})} \left[\partial_{\sigma^{(i)}} \ell\left(\cdots, \wt{\sigma}_{\Pi}(\Mat{x}; \Mat{\vartheta}^{(i)}), \cdots\right) \nabla \sigma_{\Pi}(\Mat{x}; \Mat{\vartheta}^{(i)}_j)  \right] \\
&= w^{(i)}_j \mean_{\Pi, \Mat{x} \sim \Set{D}(\Set{X})} \left[\partial^2_{\sigma^{(i)}\sigma^{(i)}} \ell\left(\cdots, \wt{\sigma}_{\Pi}(\Mat{x}; \Mat{\vartheta}^{(i)}), \cdots\right) \nabla \sigma_{\Pi}(\Mat{x}; \Mat{\vartheta}^{(i)}_j) \partial_{\Mat{\vartheta}^{(i)}_{j'}} \wt{\sigma}_{\Pi}(\Mat{x}; \Mat{\vartheta}^{(i')})^{\top} \right] \\
&\quad + w^{(i)}_j \mean_{\Pi, \Mat{x} \sim \Set{D}(\Set{X})} \left[\partial_{\sigma^{(i)}} \ell\left(\cdots, \wt{\sigma}_{\Pi}(\Mat{x}; \Mat{\vartheta}^{(i)}), \cdots\right) \partial_{\Mat{\vartheta}^{(i)}_{j'}} \nabla \sigma_{\Pi}(\Mat{x}; \Mat{\vartheta}^{(i)}_j) \right] \\
&= w^{(i)}_j w^{(i)}_{j'} \mean_{\Pi, \Mat{x} \sim \Set{D}(\Set{X})} \left[\partial^2_{\sigma^{(i)}\sigma^{(i)}} \ell\left(\cdots, \wt{\sigma}_{\Pi}(\Mat{x}; \Mat{\vartheta}^{(i)}), \cdots\right) \nabla \sigma_{\Pi}(\Mat{x}; \Mat{\vartheta}^{(i)}_j) \nabla \sigma_{\Pi}(\Mat{x}; \Mat{\vartheta}^{(i)}_{j'})^{\top} \right] \\
&\quad + w^{(i)}_j \mean_{\Pi, \Mat{x} \sim \Set{D}(\Set{X})} \left[\partial_{\sigma^{(i)}} \ell\left(\cdots, \wt{\sigma}_{\Pi}(\Mat{x}; \Mat{\vartheta}^{(i)}), \cdots\right) \partial_{\Mat{\vartheta}^{(i)}_{j'}} \nabla \sigma_{\Pi}(\Mat{x}; \Mat{\vartheta}^{(i)}_j) \right].
\end{align*}
Now we consider two scenarios:
\begin{enumerate}
\item When $j = j'$, $\partial_{\Mat{\vartheta}^{(i)}_{j'}} \nabla \sigma_{\Pi}(\Mat{x}; \Mat{\vartheta}^{(i)}_j) = \nabla^2 \sigma_{\Pi}(\Mat{x}; \Mat{\vartheta}^{(i)}_j)$, and then
\begin{align}
\left.\partial^2_{\Mat{\vartheta}^{(i)}_j\Mat{\vartheta}^{(i)}_{j}} \cL(\Mat{\theta}^{(\setminus i)}, \Mat{\vartheta}^{(i)}, \Mat{w}^{(i)})\right\vert_{\epsilon=0} &= w^{(i)2}_j \mean_{\Pi, \Mat{x} \sim \Set{D}(\Set{X})} \left[\partial^2_{\sigma^{(i)}\sigma^{(i)}} \ell\left(\cdots, \wt{\sigma}_{\Pi}(\Mat{x}; \Mat{\vartheta}^{(i)}), \cdots\right) \nabla \sigma_{\Pi}(\Mat{x}; \Mat{\vartheta}^{(i)}_j) \nabla \sigma_{\Pi}(\Mat{x}; \Mat{\vartheta}^{(i)}_j)^{\top} \right] \nonumber \\
&\quad + w^{(i)}_j \mean_{\Pi, \Mat{x} \sim \Set{D}(\Set{X})} \left[\partial_{\sigma^{(i)}} \ell\left(\cdots, \wt{\sigma}_{\Pi}(\Mat{x}; \Mat{\vartheta}^{(i)}), \cdots\right) \nabla^2 \sigma_{\Pi}(\Mat{x}; \Mat{\vartheta}^{(i)}_j) \right] \label{eqn:d2_vartheta_chain_split_one_same_ij}
\end{align}

\item When $j \ne j'$, $\partial_{\Mat{\vartheta}^{(i)}_{j'}} \nabla \sigma_{\Pi}(\Mat{x}; \Mat{\vartheta}^{(i)}_j) = \Mat{0}$, and thus
\begin{align}
\left.\partial^2_{\Mat{\vartheta}^{(i)}_j\Mat{\vartheta}^{(i)}_{j'}} \cL(\Mat{\theta}^{(\setminus i)}, \Mat{\vartheta}^{(i)}, \Mat{w}^{(i)})\right\vert_{\epsilon=0} &= w^{(i)}_j w^{(i)}_{j'} \mean_{\Pi, \Mat{x} \sim \Set{D}(\Set{X})} \left[\partial^2_{\sigma^{(i)}\sigma^{(i)}} \ell\left(\cdots, \wt{\sigma}_{\Pi}(\Mat{x}; \Mat{\vartheta}^{(i)}), \cdots\right) \nabla \sigma_{\Pi}(\Mat{x}; \Mat{\vartheta}^{(i)}_j) \nabla \sigma_{\Pi}(\Mat{x}; \Mat{\vartheta}^{(i)}_{j'})^{\top} \right] \label{eqn:d2_vartheta_chain_split_one_same_i}
\end{align}
\end{enumerate}
Using this fact again: $\Mat{\vartheta}^{(i)}_j = \Mat{\theta}^{(i)}$ and $\wt{\sigma}_{\Pi}(\Mat{x}; \Mat{\vartheta}^{(i)}) = \sum_{j=1}^{m_i} w^{(i)}_j \sigma_{\Pi}(\Mat{x}; \Mat{\theta}^{(i)}) = \sigma_{\Pi}(\Mat{x}; \Mat{\theta}^{(i)})$ when $\epsilon = 0$, Eq. \ref{eqn:d2_vartheta_chain_split_one_same_ij} becomes:
\begin{align*}
\partial^2_{\Mat{\vartheta}^{(i)}_j\Mat{\vartheta}^{(i)}_{j}} \cL(\Mat{\theta}^{(\setminus i)}, \Mat{\vartheta}^{(i)}, \Mat{w}^{(i)}) &= w^{(i)2}_j \mean_{\Pi, \Mat{x} \sim \Set{D}(\Set{X})} \left[\partial^2_{\sigma^{(i)}\sigma^{(i)}} \ell\left(\cdots, \sigma_{\Pi}(\Mat{x}; \Mat{\theta}^{(i)}), \cdots\right) \nabla \sigma_{\Pi}(\Mat{x}; \Mat{\theta}^{(i)}) \nabla \sigma_{\Pi}(\Mat{x}; \Mat{\theta}^{(i)})^{\top} \right] \\
&\quad + w^{(i)}_j \mean_{\Pi, \Mat{x} \sim \Set{D}(\Set{X})} \left[\partial_{\sigma^{(i)}} \ell\left(\cdots, \sigma_{\Pi}(\Mat{x}; \Mat{\theta}^{(i)}), \cdots\right) \nabla^2 \sigma_{\Pi}(\Mat{x}; \Mat{\theta}^{(i)}) \right] \\
&= w^{(i)2}_j \Mat{T}^{(i)}(\Mat{\theta}) + w^{(i)}_j \Mat{S}^{(i)}(\Mat{\theta}),
\end{align*}
and Eq. \ref{eqn:d2_vartheta_chain_split_one_same_i} can be simplified as:
\begin{align*}
\partial^2_{\Mat{\vartheta}^{(i)}_j\Mat{\vartheta}^{(i)}_{j'}} \cL(\Mat{\theta}^{(\setminus i)}, \Mat{\vartheta}^{(i)}, \Mat{w}^{(i)}) &= w^{(i)}_j w^{(i)}_{j'} \mean_{\Pi, \Mat{x} \sim \Set{D}(\Set{X})} \left[\partial^2_{\sigma^{(i)}\sigma^{(i)}} \ell\left(\cdots, \sigma_{\Pi}(\Mat{x}; \Mat{\theta}^{(i)}), \cdots\right) \nabla \sigma_{\Pi}(\Mat{x}; \Mat{\theta}^{(i)}) \nabla \sigma_{\Pi}(\Mat{x}; \Mat{\theta}^{(i)})^{\top} \right] \\
&= w^{(i)}_j w^{(i)}_{j'} \Mat{T}^{(i)}(\Mat{\theta}),
\end{align*}
both as desired.
\end{proof}

\begin{lemma} \label{lem:split_one_vs_split_all}
Assume $\cL(\Mat{\vartheta}, \Mat{w})$ has bounded third-order derivatives with respect to $\Mat{\vartheta}$, then we have
\begin{align*}
\left(\cL(\Mat{\vartheta}, \Mat{w}) - \cL(\Mat{\theta})\right) =
\sum_{i=1}^{n} \left(\cL(\Mat{\theta}^{(\setminus i)}, \Mat{\vartheta}^{(i)}, \Mat{w}^{(i)}) - \cL(\Mat{\theta}) \right) + \frac{\epsilon^2}{2} \sum_{\substack{i, i' \in [n] \\ i \ne i'}} \Mat{\mu}^{(i)\top} \partial^2_{\Mat{\theta}^{(i)} \Mat{\theta}^{(i')}} \cL(\Mat{\theta}) \Mat{\mu}^{(i')} + \bigO(\epsilon^3),
\end{align*}
where $\cL(\Mat{\theta}^{(\setminus i)}, \Mat{\vartheta}^{(i)}, \Mat{w}^{(i)})$ and $\Mat{\mu}^{(i)}$ are as defined in Theorem \ref{thm:sec_ord_approx}.
\end{lemma}
\begin{proof}
Define an auxiliary function:
\begin{align*}
F(\epsilon) = \left(\cL(\Mat{\vartheta}, \Mat{w}) - \cL(\Mat{\theta})\right) -
\sum_{i=1}^{n} \left(\cL(\Mat{\theta}^{(\setminus i)}, \Mat{\vartheta}^{(i)}, \Mat{w}^{(i)}) - \cL(\Mat{\theta}) \right).
\end{align*}
Note that $F(\epsilon)$ also has bounded third-order derivatives.
Hence, by Taylor expansion:
\begin{align}
F(\epsilon) = F(0) + \epsilon \frac{d}{d \epsilon} F(0) + \frac{\epsilon^2}{2} \frac{d^2}{d \epsilon^2}F(0) + \bigO(\epsilon^3). \label{eqn:F_taylor_expand}
\end{align}
Compute the first-order derivatives of $F$ via path derivatives, we can derive
\begin{align}
&\frac{d}{d \epsilon} \left(\cL(\Mat{\vartheta}, \Mat{w}) - \cL(\Mat{\theta})\right) = \sum_{i=1}^{n}\sum_{j=1}^{m_i} \partial_{\Mat{\vartheta}^{(i)}_j} \cL(\Mat{\vartheta}, \Mat{w})^\top \frac{d \Mat{\vartheta}^{(i)}_j}{d \epsilon}
= \sum_{i=1}^{n}\sum_{j=1}^{m_i} \partial_{\Mat{\vartheta}^{(i)}_j} \cL(\Mat{\vartheta}, \Mat{w})^\top \Mat{\Delta}^{(i)}_j, \label{eqn:d_eps_split_all} 
\end{align}
and for every $i \in [n]$:
\begin{align}
\frac{d}{d \epsilon} \left(\cL(\Mat{\theta}^{(\setminus i)}, \Mat{\vartheta}^{(i)}, \Mat{w}^{(i)}) - \cL(\Mat{\theta})\right) &= \sum_{j=1}^{m_i} \partial_{\Mat{\vartheta}^{(i)}_j} \cL(\Mat{\theta}^{(\setminus i)}, \Mat{\vartheta}^{(i)}, \Mat{w}^{(i)})^\top \frac{d \Mat{\vartheta}^{(i)}_j}{d \epsilon} \nonumber \\
&= \sum_{j=1}^{m_i} \partial_{\Mat{\vartheta}^{(i)}_j} \cL(\Mat{\theta}^{(\setminus i)}, \Mat{\vartheta}^{(i)}, \Mat{w}^{(i)})^\top \Mat{\Delta}^{(i)}_j, \label{eqn:d_eps_split_one}
\end{align}
By Lemma \ref{lem:derivatives_split_all} and Lemma \ref{lem:derivatives_split_one}, $\left.\partial_{\Mat{\vartheta}^{(i)}_j} \cL(\Mat{\theta}^{(\setminus i)}, \Mat{\vartheta}^{(i)}, \Mat{w}^{(i)}) \right\vert_{\epsilon=0} = \left.\partial_{\Mat{\vartheta}^{(i)}_j} \cL(\Mat{\vartheta}, \Mat{w})\right\vert_{\epsilon=0}$, hence combining Eq. \ref{eqn:d_eps_split_all} and \ref{eqn:d_eps_split_one}:
\begin{align}
\frac{d}{d\epsilon}F(0) = \left.\left[\sum_{i=1}^{n}\sum_{j=1}^{m_i} \partial_{\Mat{\vartheta}^{(i)}_j} \cL(\Mat{\vartheta}, \Mat{w})^\top \Mat{\Delta}^{(i)}_j - \sum_{i=1}^{n}\sum_{j=1}^{m_i} \partial_{\Mat{\vartheta}^{(i)}_j} \cL(\Mat{\theta}^{(\setminus i)}, \Mat{\vartheta}^{(i)}, \Mat{w}^{(i)})^\top \Mat{\Delta}^{(i)}_j  \right]\right\vert_{\epsilon=0}
= 0. \label{eqn:d_F_epsilon}
\end{align}
We can also compute the second-order derivatives via path derivatives:
\begin{align}
\frac{d^2}{d \epsilon^2} \left(\cL(\Mat{\vartheta}, \Mat{w}) - \cL(\Mat{\theta})\right)  &= \frac{d}{d \epsilon} \sum_{i=1}^{n}\sum_{j=1}^{m_i} \partial_{\Mat{\vartheta}^{(i)}_j} \cL(\Mat{\vartheta}, \Mat{w})^\top \Mat{\Delta}^{(i)}_j  \nonumber \\
&= \sum_{i=1}^{n} \sum_{i'=1}^{n} \sum_{j=1}^{m_i} \sum_{j'=1}^{m_{i'}} \Mat{\Delta}^{(i)\top}_j \partial^2_{\Mat{\vartheta}^{(i)}_j\Mat{\vartheta}^{(i')}_{j'}} \cL(\Mat{\vartheta}, \Mat{w})  \frac{d \Mat{\vartheta}^{(i')}_{j'}}{d \epsilon} \nonumber \\
&= \sum_{i=1}^{n} \sum_{i'=1}^{n} \sum_{j=1}^{m_i} \sum_{j'=1}^{m_{i'}} \Mat{\Delta}^{(i)\top}_j \partial^2_{\Mat{\vartheta}^{(i)}_j\Mat{\vartheta}^{(i')}_{j'}} \cL(\Mat{\vartheta}, \Mat{w}) \Mat{\Delta}^{(i')}_{j'},
\label{eqn:d2_eps_split_all}
\end{align}
and similarly for every $i \in [n]$,
\begin{align}
\frac{d^2}{d \epsilon^2} \left(\cL(\Mat{\theta}^{(\setminus i)}, \Mat{\vartheta}^{(i)}, \Mat{w}^{(i)}) - \cL(\Mat{\theta})\right)  &= \frac{d}{d \epsilon} \sum_{j=1}^{m_i} \partial_{\Mat{\vartheta}^{(i)}_j} \cL(\Mat{\theta}^{(\setminus i)}, \Mat{\vartheta}^{(i)}, \Mat{w}^{(i)})^\top \Mat{\Delta}^{(i)}_j  \nonumber \\
&= \sum_{j=1}^{m_i} \sum_{j'=1}^{m_i} \Mat{\Delta}^{(i)\top}_j \partial^2_{\Mat{\vartheta}^{(i)}_j\Mat{\vartheta}^{(i)}_{j'}} \cL(\Mat{\theta}^{(\setminus i)}, \Mat{\vartheta}^{(i)}, \Mat{w}^{(i)}) \frac{d \Mat{\vartheta}^{(i)}_{j'}}{d \epsilon} \nonumber \\
&= \sum_{j=1}^{m_i} \sum_{j'=1}^{m_i} \Mat{\Delta}^{(i)^\top}_j \partial^2_{\Mat{\vartheta}^{(i)}_j\Mat{\vartheta}^{(i)}_{j'}} \cL(\Mat{\theta}^{(\setminus i)}, \Mat{\vartheta}^{(i)}, \Mat{w}^{(i)}) \Mat{\Delta}^{(i)}_{j'}.
\label{eqn:d2_eps_split_one}
\end{align}
By Lemma \ref{lem:derivatives_split_all} and Lemma \ref{lem:derivatives_split_one}, $\left.\partial^2_{\Mat{\vartheta}^{(i)}_j\Mat{\vartheta}^{(i)}_{j'}} \cL(\Mat{\theta}^{(\setminus i)}, \Mat{\vartheta}^{(i)}, \Mat{w}^{(i)}) \right\vert_{\epsilon=0} = \left.\partial^2_{\Mat{\vartheta}^{(i)}_j\Mat{\vartheta}^{(i)}_{j'}} \cL(\Mat{\vartheta}, \Mat{w})\right\vert_{\epsilon=0}$ for any $i \in [n]$ and $j, j' \in [m_i]$, hence we can cancel all terms in Eq. \ref{eqn:d2_eps_split_one} by:
\begin{align}
\frac{d^2}{d\epsilon^2}F(0) &= \left.\left[\sum_{i,i' \in [n]} \sum_{j=1}^{m_i} \sum_{j'=1}^{m_{i'}} \Mat{\Delta}^{(i)\top}_j \partial^2_{\Mat{\vartheta}^{(i)}_j\Mat{\vartheta}^{(i')}_{j'}} \cL(\Mat{\vartheta}, \Mat{w}) \Mat{\Delta}^{(i')}_{j'} - \sum_{i=1}^{n}\sum_{j=1}^{m_i} \sum_{j'=1}^{m_i} \Mat{\Delta}^{(i)^\top}_j \partial^2_{\Mat{\vartheta}^{(i)}_j\Mat{\vartheta}^{(i)}_{j'}} \cL(\Mat{\theta}^{(\setminus i)}, \Mat{\vartheta}^{(i)}, \Mat{w}^{(i)}) \Mat{\Delta}^{(i)}_{j'}  \right]\right\vert_{\epsilon=0} \nonumber \\
&= \left.\left[\sum_{i=1}^{n} \sum_{j=1}^{m_i} \sum_{j'=1}^{m_{i'}} \Mat{\Delta}^{(i)\top}_j \left(\partial^2_{\Mat{\vartheta}^{(i)}_j\Mat{\vartheta}^{(i)}_{j'}} \cL(\Mat{\vartheta}, \Mat{w}) - \partial^2_{\Mat{\vartheta}^{(i)}_j\Mat{\vartheta}^{(i)}_{j'}} \cL(\Mat{\theta}^{(\setminus i)}, \Mat{\vartheta}^{(i)}, \Mat{w}^{(i)}) \right) \Mat{\Delta}^{(i)}_{j'}  \right]\right\vert_{\epsilon=0} \nonumber \\
&\quad +\left.\left[\sum_{i \ne i'} \sum_{j=1}^{m_i} \sum_{j'=1}^{m_{i'}} \Mat{\Delta}^{(i)\top}_j \partial^2_{\Mat{\vartheta}^{(i)}_j\Mat{\vartheta}^{(i')}_{j'}} \cL(\Mat{\vartheta}, \Mat{w}) \Mat{\Delta}^{(i')}_{j'}  \right]\right\vert_{\epsilon=0} \nonumber \\
&= \sum_{i \ne i'} \sum_{j=1}^{m_i} \sum_{j'=1}^{m_{i'}} w^{(i)}_j w^{(i')}_{j'} \Mat{\Delta}^{(i)\top}_j \partial^2_{\Mat{\vartheta}^{(i)}_j\Mat{\vartheta}^{(i')}_{j'}} \cL(\Mat{\theta}) \Mat{\Delta}^{(i')}_{j'} \nonumber \\
&= \sum_{i \ne i'} \sum_{j=1}^{m_i} \sum_{j'=1}^{m_{i'}} \Mat{\mu}^{(i)\top} \partial^2_{\Mat{\vartheta}^{(i)}_j\Mat{\vartheta}^{(i')}_{j'}} \cL(\Mat{\theta}) \Mat{\mu}^{(i')}, \label{eqn:d2_F_epsilon}
\end{align}
where we use Eq. \ref{eqn:derivatives_split_all_cross} in Lemma \ref{lem:derivatives_split_all} for the last second equality, and we use the fact: $\sum_{j=1}^{m_i} w^{(i)}_j \Mat{\Delta}^{(i)}_j = \Mat{\mu}^{(i)}$ to get the last equality.
Merging Eq. \ref{eqn:F_taylor_expand}, \ref{eqn:d_F_epsilon}, \ref{eqn:d2_F_epsilon}, we obtain the result as desired. 
\end{proof}

\begin{lemma} \label{lem:loss_decomp_split_one}
Assume $\cL(\Mat{\vartheta}, \Mat{w})$ has bounded third-order derivatives with respect to $\Mat{\vartheta}$, then we have
\begin{align*}
\cL(\Mat{\theta}^{(\setminus i)}, \Mat{\vartheta}^{(i)}, \Mat{w}^{(i)}) - \cL(\Mat{\theta}) &= \epsilon \partial_{\Mat{\theta}^{(i)}} \cL(\Mat{\theta})^{\top} \Mat{\mu}^{(i)} + \frac{\epsilon^2}{2} \Mat{\mu}^{(i)\top} 
\partial^2_{\Mat{\theta}^{(i)} \Mat{\theta}^{(i)}} \cL(\Mat{\theta}) \Mat{\mu}^{(i)} \\
& \quad + \frac{\epsilon^2}{2} \sum_{j=1}^{m_i} w^{(i)}_{j} \Mat{\delta}^{(i)\top}_j \Mat{S}^{(i)}(\Mat{\theta}) \Mat{\delta}^{(i)}_j + \bigO(\epsilon^3).
\end{align*}
\end{lemma}
\begin{proof}
Let $\wb{\Mat{\theta}}^{(i)} = \{ \Mat{\theta}^{(i)}, \cdots, \Mat{\theta}^{(i)}\}$ such that $\lvert \wb{\Mat{\theta}}^{(i)} \rvert = m_i$.
This is we split the $i$-th Gaussian into $m_i$ off-springs with parameters identical to the original one, or namely we let $\epsilon = 0$.
If we replace $\Mat{\vartheta}^{(i)}$ with $\wb{\Mat{\theta}}^{(i)}$, it holds that: 
\begin{align*}
\cL(\Mat{\theta}^{(\setminus i)}, \wb{\Mat{\theta}}^{(i)}, \Mat{w}^{(i)}) &= \mean_{\Pi, \Mat{x} \sim \Set{D}(\Set{X})} \left[\ell\left( \sigma_{\Pi}(\Mat{x}; \Mat{\theta}^{(1)}), \cdots, \sum_{j=1}^{m_i} w^{(i)}_j \sigma_{\Pi}(\Mat{x}; \Mat{\theta}^{(i)}), \cdots, \sigma_{\Pi}(\Mat{x}; \Mat{\theta}^{(n)}) \right) \right] \\
&= \mean_{\Pi, \Mat{x} \sim \Set{D}(\Set{X})} \left[\ell\left( \sigma_{\Pi}(\Mat{x}; \Mat{\theta}^{(1)}), \cdots, \sigma_{\Pi}(\Mat{x}; \Mat{\theta}^{(i)}), \cdots, \sigma_{\Pi}(\Mat{x}; \Mat{\theta}^{(n)}) \right) \right] = \cL(\Mat{\theta}),
\end{align*}
and
\begin{align*}
\partial_{\Mat{\vartheta}^{(i)}_j} \cL(\Mat{\theta}^{(\setminus i)}, \wb{\Mat{\theta}}^{(i)}, \Mat{w}^{(i)}) &= \partial_{\Mat{\vartheta}^{(i)}_j} \left.\cL(\Mat{\theta}^{(\setminus i)}, \Mat{\vartheta}^{(i)}, \Mat{w}^{(i)}) \right\vert_{\epsilon=0}.
\end{align*}
By Taylor expansion,
\begin{align*}
& \cL(\Mat{\theta}^{(\setminus i)}, \Mat{\vartheta}^{(i)}, \Mat{w}^{(i)}) - \cL(\Mat{\theta}) = \cL(\Mat{\theta}^{(\setminus i)}, \Mat{\vartheta}^{(i)}, \Mat{w}^{(i)}) - \cL(\Mat{\theta}^{(\setminus i)}, \wb{\Mat{\theta}}^{(i)}, \Mat{w}^{(i)}) \\
&= \sum_{j=1}^{m_i} \left.\cL(\Mat{\theta}^{(\setminus i)}, \Mat{\vartheta}^{(i)}, \Mat{w}^{(i)}) \right\vert_{\epsilon=0}^{\top} (\Mat{\vartheta}^{(i)}_{j} - \Mat{\theta}^{(i)}) \\
& \quad + \sum_{j, j' \in [m_i]} (\Mat{\vartheta}^{(i)}_{j} - \Mat{\theta}^{(i)})^{\top}\left.\partial_{\Mat{\vartheta}^{(i)}_j\Mat{\vartheta}^{(i)}_{j'}}\cL(\Mat{\theta}^{(\setminus i)}, \Mat{\vartheta}^{(i)}, \Mat{w}^{(i)}) \right\vert_{\epsilon=0} (\Mat{\vartheta}^{(i)}_{j'} - \Mat{\theta}^{(i)}) + \bigO(\epsilon^3).
\end{align*}
By Lemma \ref{lem:derivatives_split_one} and \ref{lem:derivatives_orig}:
\begin{align*}
& \cL(\Mat{\theta}^{(\setminus i)}, \Mat{\vartheta}^{(i)}, \Mat{w}^{(i)}) - \cL(\Mat{\theta}) \\
&= \epsilon \sum_{j=1}^{m_i} w^{(i)}_j \partial_{\Mat{\theta}^{(i)}} \cL(\Mat{\theta})^{\top} \Mat{\Delta}^{(i)}_{j} + \frac{\epsilon^2}{2} \sum_{j=1}^{m_i} \Mat{\Delta}^{(i)\top}_{j} \left(w^{(i)}_j \Mat{S}^{(i)}(\Mat{\theta}) + w^{(i)^2}_j \Mat{T}^{(i)}(\Mat{\theta})\right)\Mat{\Delta}^{(i)}_{j} \\
&\quad + \frac{\epsilon^2}{2} \sum_{j, j' \in [m_i], j \ne j'} w^{(i)}_j w^{(i)}_{j'} \Mat{\Delta}^{(i)\top}_{j}  \Mat{T}^{(i)}(\Mat{\theta})\Mat{\Delta}^{(i)}_{j'} + \bigO(\epsilon^3) \\
&= \epsilon \sum_{j=1}^{m_i} w^{(i)}_j \partial_{\Mat{\theta}^{(i)}} \cL(\Mat{\theta})^{\top} \Mat{\Delta}^{(i)}_{j} + \frac{\epsilon^2}{2} \sum_{j=1}^{m_i} w^{(i)}_j 
\Mat{\Delta}^{(i)\top}_{j} \Mat{S}^{(i)}(\Mat{\theta}) \Mat{\Delta}^{(i)}_{j} \\
& \quad + \frac{\epsilon^2}{2} \sum_{j, j' \in [m_i]} w^{(i)}_j w^{(i)}_{j'} \Mat{\Delta}^{(i)\top}_{j}  \Mat{T}^{(i)}(\Mat{\theta})\Mat{\Delta}^{(i)}_{j} + \bigO(\epsilon^3) \\
&= \epsilon \partial_{\Mat{\theta}^{(i)}} \cL(\Mat{\theta})^{\top} \Mat{\mu}^{(i)} + \frac{\epsilon^2}{2} \sum_{j=1}^{m_i} w^{(i)}_j \Mat{\Delta}^{(i)\top}_{j} \Mat{S}^{(i)}(\Mat{\theta}) \Mat{\Delta}^{(i)}_{j} + \Mat{\mu}^{(i)^\top}  \Mat{T}^{(i)}(\Mat{\theta}) \Mat{\mu}^{(i)} + \bigO(\epsilon^3) \\
&= \epsilon \partial_{\Mat{\theta}^{(i)}} \cL(\Mat{\theta})^{\top} \Mat{\mu}^{(i)} + \frac{\epsilon^2}{2} \Mat{\mu}^{(i)^\top} \left( \Mat{S}^{(i)}(\Mat{\theta}) + \Mat{T}^{(i)}(\Mat{\theta}) \right) \Mat{\mu}^{(i)} \\
& \quad + \frac{\epsilon^2}{2} \sum_{j=1}^{m_i} w^{(i)}_j \left(\Mat{\Delta}^{(i)\top}_{j} \Mat{S}^{(i)}(\Mat{\theta}) \Mat{\Delta}^{(i)}_{j} - \Mat{\mu}^{(i)\top} \Mat{S}^{(i)}(\Mat{\theta}) \Mat{\mu}^{(i)} \right) + \bigO(\epsilon^3) \\
&= \epsilon \partial_{\Mat{\theta}^{(i)}} \cL(\Mat{\theta})^{\top} \Mat{\mu}^{(i)} + \frac{\epsilon^2}{2} \Mat{\mu}^{(i)^\top}  \partial^2_{\Mat{\theta}^{(i)}\Mat{\theta}^{(i)}} \cL(\Mat{\theta}) \Mat{\mu}^{(i)} + \frac{\epsilon^2}{2} \sum_{j=1}^{m_i} w^{(i)}_j \left(\Mat{\Delta}^{(i)\top}_{j} \Mat{S}^{(i)}(\Mat{\theta}) \Mat{\Delta}^{(i)}_{j} - \Mat{\mu}^{(i)\top} \Mat{S}^{(i)}(\Mat{\theta}) \Mat{\mu}^{(i)} \right) + \bigO(\epsilon^3).
\end{align*}
Finally, we conclude the proof by showing that:
\begin{align*}
&\sum_{j=1}^{m_i} w^{(i)}_j \left(\Mat{\Delta}^{(i)\top}_{j} \Mat{S}^{(i)}(\Mat{\theta}) \Mat{\Delta}^{(i)}_{j} - \Mat{\mu}^{(i)\top} \Mat{S}^{(i)}(\Mat{\theta}) \Mat{\mu}^{(i)} \right) \\
&= \sum_{j=1}^{m_i} w^{(i)}_j \Mat{\Delta}^{(i)\top}_{j} \Mat{S}^{(i)}(\Mat{\theta}) \Mat{\Delta}^{(i)}_{j} + \Mat{\mu}^{(i)\top} \Mat{S}^{(i)}(\Mat{\theta}) \Mat{\mu}^{(i)} - 2\left(\sum_{j=1}^{m_i} w^{(i)}_j \Mat{\Delta}^{(i)}_{j} \right)^\top \Mat{S}^{(i)}(\Mat{\theta}) \Mat{\mu}^{(i)} \\
&= \sum_{j=1}^{m_i} w^{(i)}_j \left((\Mat{\Delta}^{(i)}_{j} - \Mat{\mu}^{(i)})^\top \Mat{S}^{(i)}(\Mat{\theta}) (\Mat{\Delta}^{(i)}_{j} - \Mat{\mu}^{(i)}) \right) = \sum_{j=1}^{m_i} w^{(i)}_j \Mat{\delta}^{(i)^\top}_{j} \Mat{S}^{(i)}(\Mat{\theta}) \Mat{\delta}^{(i)}_{j}.
\end{align*}
\end{proof}

\subsection{Deriving Hessian of Gaussian}
\label{sec:hessian}
In Sec. \ref{sec:impl}, we discussed that SteepGS requires the computation of Hessian matrices for $\sigma_{\Pi}(\Mat{x}; \Mat{\theta})$.
We make the following simplifications:
\textit{(i)} We only consider position parameters as the optimization variable when computing the steepest descent directions.
\textit{(ii)} Although other variables may have a dependency on the mean parameters, \eg the projection matrix and view-dependent RGB colors, we break this dependency for ease of derivation.
Now suppose we have a 3D Gaussian point with parameters $\Mat{\theta} = (\Mat{p}, \Mat{\Sigma}, o)$, where we omit RGB colors as it can be handled similarly to opacity $o$.
Given the affine transformation $\Pi: \Mat{p} \mapsto \Mat{P}\Mat{p} + \Mat{b}$ with $\Mat{P} \in \real^{2 \times 3}$ and $\Mat{b} \in \real^2$, then $\sigma_{\Pi}(\Mat{x}; \Mat{\theta})$ can be expressed as:
\begin{align*}
\sigma_{\Pi}(\Mat{x}; \Mat{\theta}) &= o \exp\left(-\frac{1}{2}(\Mat{x} - \Mat{P}\Mat{p} - \Mat{b})^\top (\Mat{P}\Mat{\Sigma}\Mat{P}^\top)^{-1} (\Mat{x} - \Mat{P}\Mat{p} - \Mat{b}) \right) = o \gauss(\Mat{x}; \Mat{P}\Mat{p} + \Mat{b}, \Mat{P}\Mat{\Sigma}\Mat{P}^\top).
\end{align*}
Its gradient can be derived as:
\begin{align*}
\nabla_{\Mat{p}} \sigma_{\Pi}(\Mat{x}; \Mat{\theta}) &= o \gauss(\Mat{x}; \Mat{P}\Mat{p} + \Mat{b}, \Mat{P}\Mat{\Sigma}\Mat{P}^\top) \nabla_{\Mat{p}}\left[-\frac{1}{2}(\Mat{x} - \Mat{P}\Mat{p} - \Mat{b})^\top (\Mat{P}\Mat{\Sigma}\Mat{P}^\top)^{-1} (\Mat{x} - \Mat{P}\Mat{p} - \Mat{b}) \right] \\
&= o \gauss(\Mat{x}; \Mat{P}\Mat{p} + \Mat{b}, \Mat{P}\Mat{\Sigma}\Mat{P}^\top) \Mat{P}^\top (\Mat{P}\Mat{\Sigma}\Mat{P}^\top)^{-1} (\Mat{x} - \Mat{P}\Mat{p} - \Mat{b}).
\end{align*}
Now we can compute the Hessian matrix as:
\begin{align*}
\nabla^2_{\Mat{p}} \sigma_{\Pi}(\Mat{x}; \Mat{\theta}) &= \Mat{P}^\top(\Mat{P}\Mat{\Sigma}\Mat{P}^\top)^{-1} (\Mat{x} - \Mat{P}\Mat{p} + \Mat{b}) \nabla_{\Mat{p}}\left[o \gauss(\Mat{x}; \Mat{P}\Mat{p} + \Mat{b}, \Mat{P}\Mat{\Sigma}\Mat{P}^\top)\right]^\top
- o \gauss(\Mat{x}; \Mat{P}\Mat{p} + \Mat{b}, \Mat{P}\Mat{\Sigma}\Mat{P}^\top) (\Mat{P}\Mat{\Sigma}\Mat{P}^\top)^{-1}\Mat{P} \\
&= o \gauss(\Mat{x}; \Mat{P}\Mat{p} + \Mat{b}, \Mat{P}\Mat{\Sigma}\Mat{P}^\top) \Mat{P}^\top (\Mat{P}\Mat{\Sigma}\Mat{P}^\top)^{-1} (\Mat{x} - \Mat{P}\Mat{p} - \Mat{b}) (\Mat{x} - \Mat{P}\Mat{p} - \Mat{b})^\top (\Mat{P}\Mat{\Sigma}\Mat{P}^\top)^{-1}\Mat{P} \\
& \quad - o \gauss(\Mat{x}; \Mat{P}\Mat{p} + \Mat{b}, \Mat{P}\Mat{\Sigma}\Mat{P}^\top) \Mat{P}^\top (\Mat{P}\Mat{\Sigma}\Mat{P}^\top)^{-1}\Mat{P} \\
&= \sigma_{\Pi}(\Mat{x}; \Mat{\theta}) \left( \Mat{P}^\top (\Mat{P}\Mat{\Sigma}\Mat{P}^\top)^{-1} (\Mat{x} - \Mat{P}\Mat{p} - \Mat{b}) (\Mat{x} - \Mat{P}\Mat{p} - \Mat{b})^\top (\Mat{P}\Mat{\Sigma}\Mat{P}^\top)^{-1}\Mat{P} - \Mat{P}^\top (\Mat{P}\Mat{\Sigma}\Mat{P}^\top)^{-1}\Mat{P} \right)
\end{align*}
as desired.

\end{document}